\documentclass[letterpaper]{article} 
\usepackage[preprint]{aaai2027}  
\usepackage[hyphens]{url}  
\usepackage{graphicx} 
\usepackage{natbib}  
\usepackage{caption} 
\usepackage{algorithm}
\usepackage{algorithmic}

\usepackage{newfloat}

\usepackage{listings}
\DeclareCaptionStyle{ruled}{labelfont=normalfont,labelsep=colon,strut=off} 
\floatstyle{ruled}
\newfloat{listing}{tb}{lst}{}
\floatname{listing}{Listing}
\usepackage{multirow} 
\usepackage{amsmath}
\usepackage{amssymb}
\usepackage{mathtools}
\usepackage{amsthm}

\usepackage[table]{xcolor} 
\usepackage{booktabs}
\usepackage{colortbl}
\usepackage{tcolorbox}
\tcbuselibrary{breakable}
\usepackage[hybrid]{markdown} 
\usepackage{subcaption}
\theoremstyle{plain}
\newtheorem{theorem}{Theorem}

\newtheorem{lemma}[theorem]{Lemma}
\newtheorem{corollary}[theorem]{Corollary}
\theoremstyle{definition}
\newtheorem{definition}[theorem]{Definition}
\newtheorem{assumption}[theorem]{Assumption}
\theoremstyle{remark}

\usepackage{booktabs}

\title{Scaling Representation Diversity: Modulated Attention and Reconstructive Regularization for Visual Grounding}
\author{
    Written by AAAI Press Staff\textsuperscript{\rm 1}\thanks{With help from the AAAI Publications Committee.}\\
    AAAI Style Contributions by Peter Patel Schneider,
    Sunil Issar,\\
    J. Scott Penberthy,
    George Ferguson,
    Hans Guesgen,
    Francisco Cruz\equalcontrib\corresponding,
    Marc Pujol-Gonzalez\equalcontrib\corresponding
}
\affiliations{
    \textsuperscript{\rm 1}Association for the Advancement of Artificial Intelligence\\

    1101 Pennsylvania Ave, NW Suite 300\\
    Washington, DC 20004 USA\\
    proceedings-questions@aaai.org
}

\title{Scaling Representation Diversity: Modulated Attention and Reconstructive Regularization for Visual Grounding}
\author {
    Junyi Hu\textsuperscript{\rm 1}
    Tian Bai\textsuperscript{\rm 2}
    Fengyi Wu\textsuperscript{\rm 2}
    Yian Huang\textsuperscript{\rm 2}
    Wei Wen\textsuperscript{\rm 3}\\
    Zaoli Li\textsuperscript{\rm 3}
    Junli Lin\textsuperscript{\rm 4}
    Xingchen Li\textsuperscript{\rm 5}
    Zhenming Peng\textsuperscript{\rm 2}
    Yi Zhang\textsuperscript{\rm 1}\corresponding
}
\affiliations {
    \textsuperscript{\rm 1}Department of Automation, Tsinghua University\\
    \textsuperscript{\rm 2}School of Information and Communication Engineering, University of Electronic Science and Technology of China\\
    \textsuperscript{\rm 3}Chinalco Digital Intelligence (Chengdu) Technology Co., Ltd.\\
    \textsuperscript{\rm 4}Linsulabs\,\,\, \textsuperscript{\rm 5}PetroChina Southwest Oil and Gas Field Company, CNPC
}

\begin{document}

\maketitle

\begin{abstract}
Referring Expression Comprehension (REC) is commonly studied under dataset-specific fine-tuning, resulting in specialist models with limited cross-dataset generalization. In this work, we revisit REC from the perspective of unified open-vocabulary grounding and identify representation degeneration as a key obstacle to scaling a single generalist model. To preserve \textbf{representation diversity}, we propose a holistic data-model co-design framework. Architecturally, we introduce the Modulated Attention-Contrastive Head (mACH) for efficient token-level vision-language alignment and a text-conditioned JEPA auxiliary stream that provides complementary gradient support to preserve alignment-active representations without inference overhead. On the data side, we introduce \texttt{Objects365-Caption}, enriching Objects365 with context-aware referring expressions for large-scale language supervision. We further provide a theoretical analysis showing that complementary gradient subspaces preserve alignment capacity and thereby scale representation diversity. Extensive experiments demonstrate that our single-checkpoint framework achieves highly competitive performance on standard REC benchmarks while exhibiting strong generalization across heterogeneous grounding datasets without benchmark-specific adaptation.
\end{abstract}
\begin{links}
    \link{Code}{https://github.com/inlmouse/MACH}
    \link{Datasets}{https://huggingface.co/datasets/EndlessnessSoul/Objects365_captions}
\end{links}

\section{Introduction}

Referring Expression Comprehension (REC) aims to localize an object in an image according to an open-ended natural language query \cite{yu2016modeling}. Existing REC methods can be broadly categorized into three computational paradigms. 
\emph{Text-conditioned spatial regression} methods repeatedly fuse language and visual features inside the backbone, leading to query-dependent computation \cite{yu2018mattnet,deng2021transvg,kamath2021mdetr}. 
\emph{Region-word dot alignment} methods instead project visual regions and text embeddings into a shared space, enabling efficient parallel grounding through matrix similarity \cite{li2022grounded,liu2024grounding,wang2022generalizing}. 
Finally, \emph{autoregressive multimodal LLMs} formulate grounding as language generation, offering remarkable reasoning ability at the cost of substantially higher inference latency \cite{chen2023shikra}.



Among these paradigms, region-word dot alignment offers an attractive balance between text-conditioned spatial regression and autoregressive MLLMs, making it well suited for resource-constrained devices and as a proposal generator for LLM-based pipelines. We therefore adopt this paradigm as our base architecture. However, scaling it to unified multi-dataset training exposes two fundamental bottlenecks: a \emph{model-level representation degeneration} where contrastive objectives induce low-rank polysemantic bottlenecks \cite{chaudhuri2025closer}, and a \emph{data-level linguistic deficit} arising from discrete category annotations. 
\begin{figure}[t]
    \centering
    \includegraphics[width=1\linewidth]{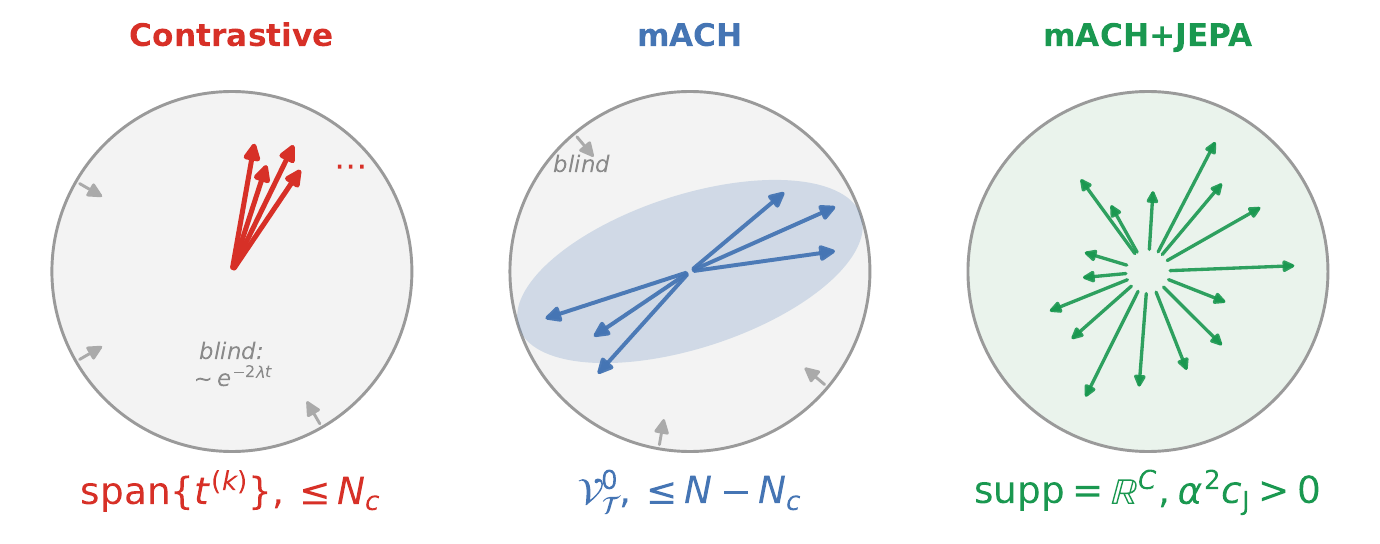}
    \caption{\textbf{Conceptual illustration of representation diversity.} Different training objectives provide supervision over different subsets of the shared visual representation. Conventional contrastive learning mainly optimizes a limited set of alignment directions, while the proposed mACH expands supervision to finer token-level interactions. The auxiliary JEPA objective further supplies complementary learning signals beyond language supervision, encouraging diverse and robust visual representations. Together, the dual-stream objective preserves a richer set of discriminative directions for unified open-vocabulary grounding.}
    \label{fig:capacity}
\end{figure}

In this work, we use \textbf{representation diversity} as the number of independent alignment-active directions preserved in the learned visual representation. As conceptually illustrated in Fig.~\ref{fig:capacity}, different training objectives preserve different subsets of the shared visual representation. Increasing representation diversity therefore requires complementary supervision that progressively expands the alignment-active subspace. To resolve these dual bottlenecks, we propose a holistic data-model co-design framework. Architecturally, we integrate the \textbf{Modulated Attention-Contrastive Head (mACH)} with an inference-free \textbf{Joint Embedding Predictive Architecture (JEPA)} auxiliary stream \cite{lecun2022path}, which complements language-conditioned gradients with visual predictive supervision to preserve representation diversity. On the data front, we construct \textbf{Objects365-Caption} (\texttt{O365-Caption}), enriching discrete object labels with context-aware descriptions to provide the linguistic diversity required for robust open-vocabulary grounding.

Our goal is to learn a unified grounding model that generalizes across heterogeneous referring expression distributions, rather than optimizing for individual benchmarks. To this end, our contributions are threefold:
\begin{itemize}
    \item We propose a data-model co-design comprising the Modulated Attention-Contrastive Head (mACH) and an inference-free JEPA auxiliary objective. We further show theoretically that the two objectives activate complementary gradient subspaces, preserving alignment capacity and scaling representation diversity.
    \item We introduce \textbf{O365-Caption}, enriching discrete detection labels with dense natural language descriptions to enhance linguistic-geometric grounding.
    \item Extensive experiments demonstrate that our unified model achieves competitive performance across standard REC benchmarks using a single static checkpoint, showing superior robustness in out-of-distribution evaluations.
\end{itemize}

\section{Related Work}

\noindent \textbf{Visual Grounding Paradigms:} 
Visual grounding has evolved from benchmark-specific specialists using cross-attention regression \cite{yu2018mattnet, deng2021transvg, xiao2024hivg} to unified open-vocabulary architectures. Modern generative MLLMs \cite{chen2023shikra, you2024ferret} offer strong open-domain reasoning via coordinate token generation, yet their high decoding latency hampers real-time edge deployment. Alternatively, discriminative region-word dot alignment models \cite{kamath2021mdetr, li2022grounded, liu2024grounding} achieve superior efficiency by decoupling visual proposal extraction from parallel text matching. However, these frameworks often rely on dataset-specific fine-tuning or purely discriminative objectives, limiting single-weight unified generalization.

\noindent \textbf{Representation Degeneration \& Predictive Regularization:} 
Discriminative and contrastive learning objectives tend to compress feature variance into low-rank, anisotropic subspaces, leading to representation collapse \cite{papyan2020prevalence, jing2021understanding, liang2022mind}. Under large-scale multimodal supervision, this degeneration severely impairs out-of-distribution generalization \cite{chaudhuri2025closer}. Although JEPA and its variants~\cite{lecun2022path, assran2023self, bardes2023v} provide non-contrastive latent prediction that preserves semantic variance, their application is largely restricted to foundation model pre-training \cite{chen2025vl, huang2026text}. Utilizing JEPA as an auxiliary regularizer to expand gradient span and prevent degeneration in unified grounding remains unexplored.

\noindent \textbf{Grounding Corpora \& Data-Model Co-Design:} 
Existing grounding benchmarks inherently trade off scale, annotation reliability, and linguistic richness \cite{krishna2017visual, peng2024grounding, rasheed2024glamm}. Large-scale detection datasets like Objects365 \cite{shao2019objects365} provide vast spatial supervision but suffer from discrete category labels. To resolve this data-level deficit, we curate \texttt{O365-Caption}, converting Objects365 category labels into dense, context-aware descriptions. This supplies the linguistic diversity required for robust open-vocabulary alignment within a scalable and accessible pipeline.

\begin{figure*}[!t]
\centering
\includegraphics[width=1\textwidth]{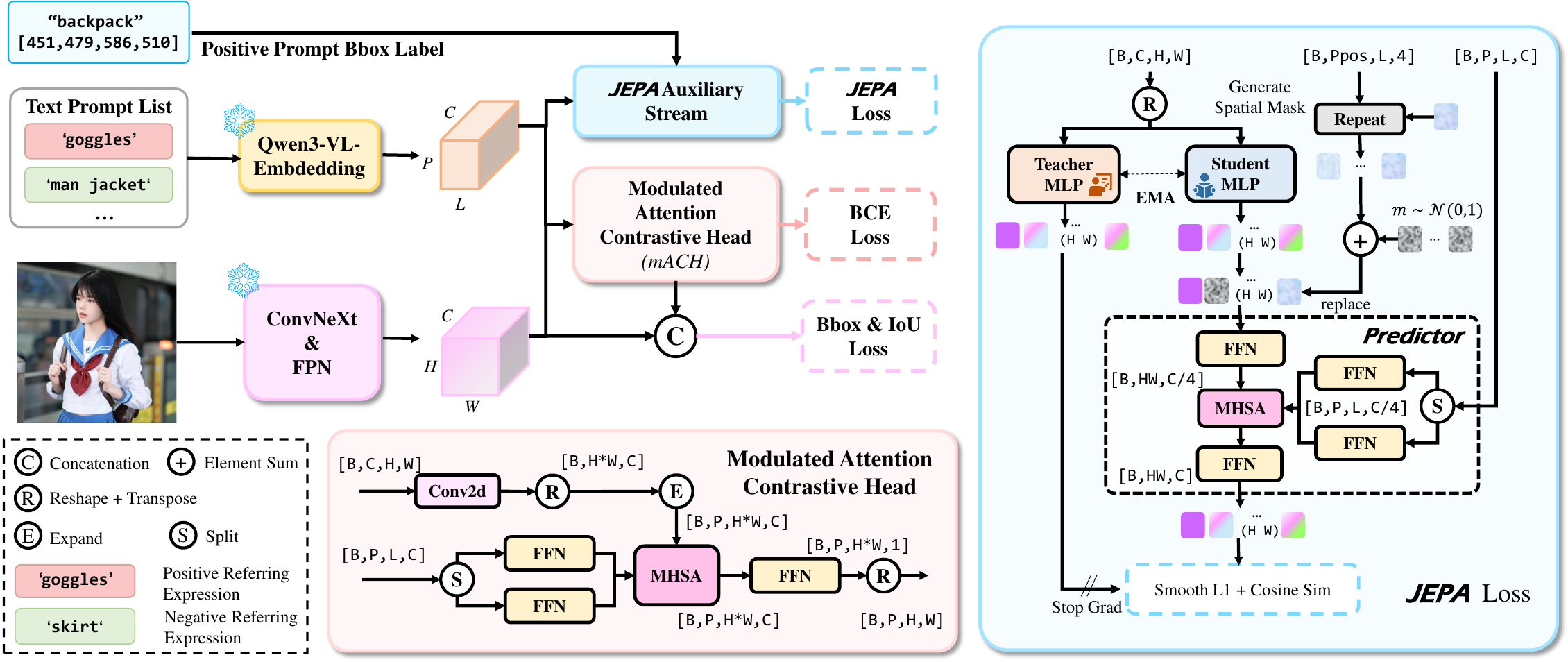}
\caption{Overview of the proposed framework (illustrated using a CNN-based detector). The visual backbone extracts visual features, while the language encoder produces text embeddings from referring expressions. The proposed Modulated Attention-Contrastive Head (mACH) performs token-level cross-attention to generate grounding predictions supervised by classification and localization losses. In parallel, a text-conditioned JEPA auxiliary branch reconstructs masked visual features under language guidance using an EMA teacher-student framework. Both branches operate on the same visual representations during training, allowing reconstructive regularization to complement discriminative grounding without introducing additional inference cost.}
\label{fig:architecture}
\end{figure*}

\section{Methodology}

The overall architecture of our framework is illustrated in Fig.~\ref{fig:architecture}. Given an input image and a set of referring expressions, the vision backbone extracts multi-scale visual features, while the language encoder produces the corresponding textual embeddings. Unlike conventional grounding frameworks that optimize only discriminative alignment, our method jointly supervises the same visual features with two complementary objectives during training: (1) a discriminative \textbf{Modulated Attention-Contrastive Head (mACH)} for open-vocabulary grounding, and (2) a reconstructive \textbf{JEPA auxiliary stream} that regularizes the shared visual representation. Since the JEPA branch is removed after training, it introduces no additional inference cost.

\subsection{Modulated Attention-Contrastive Head (mACH)}

To efficiently align visual features with multiple referring expressions concurrently, we adopt a lightweight, broadcast-based cross-attention head, termed mACH. The proposed mACH is not intended as a fundamentally new attention operator. Instead, it serves as an efficient implementation of standard cross-attention that reformulates query-text interaction into a broadcasted computation topology for unified open-vocabulary grounding. This design allows a single visual forward pass across the backbone and neck to simultaneously interact with an arbitrary number of linguistic candidates, drastically accelerating multi-query inference while remaining fully compatible with conventional legacy detector architectures.

Let
$X \in \mathbb{R}^{B \times M \times C}$
denote the flattened visual feature map extracted from one feature level of the detector, where $B$ is the image batch size, $M$ is the number of spatial locations, and $C$ is the feature dimension. Let
$W \in \mathbb{R}^{B_{nc} \times L \times C}$
denote the text embedding sequence, where $L$ is the token length and
$B_{nc}=B\times N_c$
corresponds to the expanded batch containing $N_c$ referring expressions for each image.

To avoid redundant visual computation, the visual features are broadcast along the batch dimension,
\begin{equation}
Q=\mathrm{Broadcast}(X)
\in
\mathbb{R}^{B_{nc}\times M\times C},
\label{eq:broadcast}
\end{equation}
such that each visual feature map is paired with every referring expression. Meanwhile, the textual embeddings are linearly projected to generate the corresponding keys and values $K,V=\mathrm{Linear}(W).$

Cross-modal interaction is then established through scaled dot-product attention,
\begin{equation}
O
=
\mathrm{Softmax}
\left(QK^{\top}/\sqrt{C}
\right)
V,
\label{eq:mach_attn}
\end{equation}
where $O\in\mathbb{R}^{B_{nc}\times M\times C}$ denotes the aligned visual representation. In practice, mACH is implemented using the standard multi-head attention formulation together with the variable-length implementation of FlashAttention-2 \cite{dao2023flashattention2} to efficiently eliminate computation on padded text tokens.

Finally, the aligned features are projected to grounding logits through a lightweight prediction layer. The final grounding score is computed as
\begin{equation}
S
=
\psi(O)\cdot\exp(\tau)+b,
\label{eq:final_score}
\end{equation}
where $\psi(\cdot)$ denotes a lightweight grounding head that maps the language-conditioned feature
$O\in\mathbb{R}^{B_{nc}\times M\times C}$
to grounding logits
$\psi(O)\in\mathbb{R}^{B_{nc}\times M}$. $\tau$ is a learnable logit scale, and $b$ is a learnable bias. The resulting score map is optimized using the standard binary cross-entropy objective,
\begin{equation}
\mathcal{L}_{\mathrm{mACH}}
=
\mathcal{L}_{\mathrm{BCE}}(S,Y),
\label{eq:mach_loss}
\end{equation}
where $Y$ denotes the ground-truth grounding map. Although the above description is based on our CNN implementation for clarity, the proposed mACH head is architecture-agnostic and can be readily integrated into transformer-based grounding frameworks with only minor modifications. The detailed implementation is presented in Supp. Mat. A \emph{Fig.~1}.

\subsection{JEPA Auxiliary Stream}

Although mACH provides strong discriminative supervision, optimizing only the grounding objective may gradually reduce feature diversity under large-scale heterogeneous language supervision. To alleviate this issue, we introduce a JEPA auxiliary branch that regularizes the same visual features during training through latent feature prediction. Since the branch is discarded after training, it incurs no inference overhead.

Given the visual feature map $X$, we construct an asymmetric online-target architecture consisting of a student projection head
$\mathcal{P}_{\theta}$
and an EMA teacher
$\mathcal{P}_{\mathrm{EMA}}$.
Both projection heads share the same architecture composed of lightweight
$1\times1$
convolutions,
GroupNorm,
and GELU activations.
The teacher parameters are updated using exponential moving average,
\begin{equation}
\mathcal{P}_{\mathrm{EMA}}^{(t+1)}
=
\lambda_{\mathrm{ema}}
\mathcal{P}_{\mathrm{EMA}}^{(t)}
+
(1-\lambda_{\mathrm{ema}})
\mathcal{P}_{\theta}^{(t)}.
\label{eq:ema_update}
\end{equation}

The student and teacher generate latent representations
\begin{equation}
Z_{\mathrm{stu}},
Z_{\mathrm{teach}}
\in
\mathbb{R}^{B\times C\times M},
\end{equation}
respectively.
During training, spatial regions corresponding to the ground-truth bounding boxes are randomly masked.
The masked student features are replaced with a learnable mask token,
forming
$Z_{\mathrm{stu}}^{\mathrm{masked}}$.

\begin{figure*}[!t]
\centering
\includegraphics[width=1\textwidth]{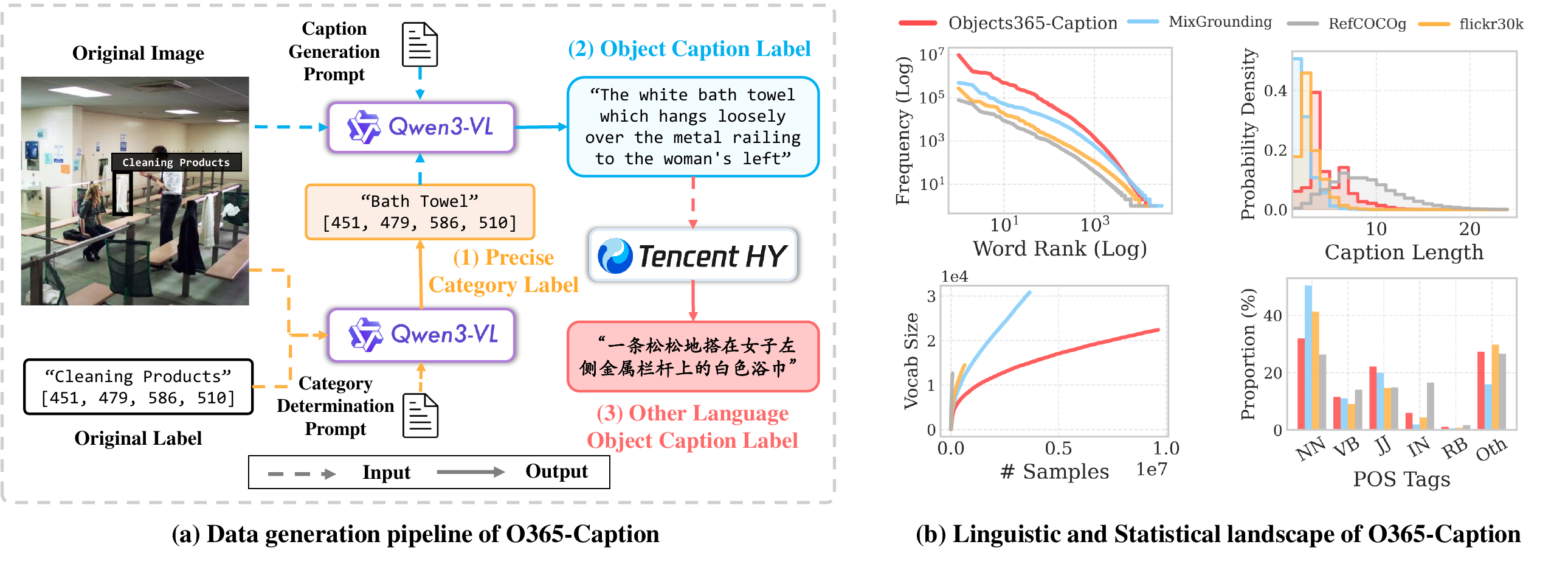}
\caption{(a) The detailed three-stage data generation pipeline of O365-Caption and (b) Comprehensive linguistic and statistical landscape of O365-Caption compared with prominent baselines. Zoom in for best view.}
\label{fig:dataset}
\end{figure*}
Unlike conventional I-JEPA, our predictor additionally receives the language embedding as contextual guidance, encouraging the masked visual regions to recover semantics relevant to the referring expression,
\begin{equation}
\hat{Z}_{\Omega}
=
\mathcal{F}_{\phi}
\left(
Z_{\mathrm{stu}}^{\mathrm{masked}},
W
\right)_m,
\quad
m\in\Omega,
\label{eq:jepa_prediction}
\end{equation}
where
$\Omega$
denotes the masked spatial locations.

The reconstruction objective is defined on normalized latent features,
\begin{equation}
\label{eq:jepa_loss}
\begin{aligned}
    \mathcal{L}_{\mathrm{JEPA}}
    & =
    \frac{1}{|\Omega|}
    \sum_{m\in\Omega}
    \left(
    1-
    \langle
    \bar{\hat z}_m,
    \bar z_{\mathrm{target},m}
    \rangle
    \right)\\
    & +
    \frac{\beta}{|\Omega|}
    \sum_{m\in\Omega}
    \mathrm{SmoothL1}
    (
    \bar{\hat z}_m,
    \bar z_{\mathrm{target},m}
    ),
\end{aligned}
\end{equation}
where
$\bar z$
denotes $\ell_2$-normalized latent features.

Finally, the overall training objective combines discriminative grounding supervision and latent reconstructive regularization,
\begin{equation}
\mathcal{L}_{\mathrm{Total}}
=
\mathcal{L}_{\mathrm{mACH}}
+
\alpha
\mathcal{L}_{\mathrm{JEPA}},
\label{eq:total_loss}
\end{equation}
where
$\alpha$
controls the weight of the auxiliary objective.
By jointly optimizing both losses on the same visual representation, the JEPA branch complements discriminative alignment with reconstructive supervision, encouraging richer feature diversity and improving the robustness of unified open-vocabulary grounding.

\subsection{O365-Caption}

To fully unlock the capacity of our alignment pipeline and anchor robust open-vocabulary scaling, we introduce \texttt{O365-Caption}. Existing unified multi-dataset training is frequently bottlenecked by a structural linguistic trade-off: massive detection corpora like Objects365 \cite{shao2019objects365} are restricted to rigid, discrete tags, whereas most REC datasets offer rich expressions but lack scale. \texttt{O365-Caption} bridges this gap by systematically upgrading discrete category labels into dense, context-aware grounding expressions, injecting essential linguistic compositionality into unified pre-training.
\begin{table}[t]
\centering
\caption{\textbf{Comprehensive statistics and linguistic diversity comparisons of open-vocabulary visual grounding and detection corpora.} We compare our curated \texttt{O365-Caption} against dominant detection, phrase grounding, and referring expression comprehension (REC) benchmarks. $\mathcal{P}_{txt}$ denotes the unique phrase size, UCR represents the Unique Caption Ratio ($\frac{\text{\# Unique Captions}}{\text{\# Total Captions}} \times 100\%$), reflecting the text-level compositionality and freedom from repetitive collapsed patterns.}
\label{tab:dataset_stats}
\small 
\setlength{\tabcolsep}{1.6pt}
\begin{tabular}{lccccc}
\toprule
\textbf{Dataset} & \textbf{\# Images} & \textbf{\# Annos} & \textbf{Avg. Len.} & \textbf{$\mathcal{P}_{txt}$} & \textbf{UCR (\%)} \\
\midrule
 \multicolumn{6}{l}{\emph{Object Detection Corpora (Discrete Categories)}} \\
COCO & 118K & 860K & 1.0 & 80 & $\sim$ 0.0 \\
Objects365 & 638K & 9.6M & 1.0 & 365 & $\sim$ 0.0 \\
\midrule
 \multicolumn{6}{l}{\emph{Phrase Grounding \& Mixed Corpora (VLM Pre-training)}} \\
Flickr30K & 31K & 638K & 2.4 & 94K & 14.7 \\
MixedGrounding & 614K & 3.7M & 1.8 & 386K & 10.5 \\
\midrule
 \multicolumn{6}{l}{\emph{Referring Expression Comprehension (REC Training Set)}} \\
RefCOCO & 20K & 121K & 3.5 & 68K & 56.0 \\
RefCOCO+ & 20K & 120K & 3.5 & 77K & \textbf{95.9} \\
RefCOCOg & 26K & 80K & \textbf{8.3} & 77K & 64.3 \\
\midrule
\textbf{O365-Caption} & \textbf{638K} & \textbf{9.6M} & 4.2 & \textbf{597K} & 6.2 \\
\bottomrule
\end{tabular}
\end{table}

As illustrated in \emph{Fig.~\ref{fig:dataset}(a)}, the dataset is constructed via an automated, three-stage generative pipeline. First, \textbf{Coarse-to-Fine Disambiguation} utilizes a lightweight MLLM (Qwen3-VL-2B) to refine generalized tags into precise taxonomies, safeguarded by a spatial coverage threshold ($\gamma < 0.05\%$) to prevent small-object hallucinations. Second, \textbf{Context-Aware Captioning} employs a powerful 32B MLLM to synthesize descriptive expressions that fuse the refined category with fine-grained visual attributes and spatial dynamics. Finally, \textbf{Cross-Lingual Extension} leverages machine translation to support multilingual evaluation beyond English-centric paradigms \cite{nogueira2025comprehension}. Despite the massive generation scale, random manual verification demonstrates an error rate strictly below 0.1\%. Exhaustive prompt templates, bounding-box guardrails, and filtering heuristics are deferred to the Supp. Mat. C.

\begin{table*}[!htbp]
\centering
  \caption{\textbf{Performance on standard referring expression comprehension (REC) benchmarks.} Results are reported as Top-1 accuracy (\%) on RefCOCO/+/g. Training datasets are abbreviated as follows: \textbf{O365}: Objects365; \textbf{OG}: Objects365 + GoldG; \textbf{GoldG-f}: GoldG after annotation refinement (deferred to the Supp. Mat.~D); \textbf{O365-C}: our proposed Objects365-Caption; \textbf{RefC}: training set of RefCOCO, RefCOCO+, and RefCOCOg~\cite{yu2016modeling}. \textbf{gRefC}: training set of gRefCOCO\cite{wang2025hierarchical}. * indicates estimated from the corresponding paper.}
  \small 
    \setlength{\tabcolsep}{3pt}
      \begin{tabular}{l|c|c|c|c|cccccccc}
        \toprule
        \multirow{2}{*}{\textbf{Method}} & \multirow{2}{*}{\textbf{Pre-Train Data}} & \multirow{2}{*}{\textbf{Input Size}} & \multirow{2}{*}{\textbf{\#Params}} & \multirow{2}{*}{\textbf{FT}} & \multicolumn{3}{c}{\textbf{RefCOCO}} & \multicolumn{3}{c}{\textbf{RefCOCO+}} & \multicolumn{2}{c}{\textbf{RefCOCOg}}  \\
        & & & & & val & testA & testB & val & testA & testB & val & test \\
        \midrule
        GLIP-T~\cite{li2022grounded} & OG & 800×1333 & 232M & N & 50.0 & 54.7 & 43.1 & 49.0 & 53.4 & 43.4 & 65.6 & 66.0\\
        GDINO-T~\cite{liu2024grounding} & OG, RefC & 800×1333 & 172M &  N & 74.0 & 74.9 & 59.3 & 66.8 & 69.9 & 56.1 & 71.1 & 72.1  \\
        PBREC-MT~\cite{zhao2024rethinking} & RefC, ReferItGame & 640x640 & 205M* &  N & 71.4 & 73.2 & 70.1 & 63.8 & 67.1 & 56.6 & 62.9 & 62.6 \\
        ExpAlign~\cite{hu2026expalign} & OG, RefC & 640x640 & 60M &  N & 51.6 & 59.3 & 47.7 & 48.9 & 47.5 & 45.5 & 65.6 & 64.0  \\
        \midrule
        LISA++-L2~\cite{yang2023lisa++} & N/A & 1024×1024 & 13B &  N & 85.9  & 88.8 &  81.7 & 74.5 & 80.6 &  68.3 & 80.1 & 81.3\\
        GSVA~\cite{xia2024gsva} & N/A & 1024×1024 & 7B &  N & 86.3 & 89.2 & 83.8 & 72.8 & 78.8 & 68.0 & 81.6 & 81.8\\
        Ours & GoldG-f, O365-C & 640x640 & 75M &  N & 85.3 & 89.0 & 82.5 & 71.8 & 78.2 & 62.7 & 76.3 &  75.8 \\
        \midrule
        MTKREC\cite{mi2024open} & novel categories RefC & 640x640 & 200M* &  Y & 81.1 & 86.8 & 73.9 & 75.0 & 80.8 & 65.5 & 74.6 &  74.7 \\
        HieA2G\cite{wang2025hierarchical} & RefC, Flickr30K, gRefC & 640x640 & 350M* &  Y & 87.8 & 90.3 & 84.0 & 80.7 & 85.6 & 72.9 & 83.7 & 83.8 \\
        GDINO-T & OG, RefC & 800×1333 & 172M & Y & 89.2 & 91.9 & 86.0 & 81.1 & 87.4 & 74.7 & 84.2 & 84.9  \\
        PropVG~\cite{dai2025propvg} & N/A & 800×1333 & 490M & Y & 89.0 &  91.6 & 85.7 & 83.7 & 88.0 & 76.6 & 83.5 & 84.4  \\
        Ours & GoldG-f, RefC, O365-C & 640x640 & 75M &  Y & 91.7 & 93.0 & 90.2 & 83.5 & 87.5 & 76.9 & 85.1 & 86.0 \\
        \bottomrule
      \end{tabular}
\label{tab:refcoco}
\end{table*}

This curation replaces legacy labels with 9.6M open-vocabulary descriptions across 638K images. As analyzed in \emph{Table~\ref{tab:dataset_stats}} and \emph{Fig.~\ref{fig:dataset}(b)}, \texttt{O365-Caption} exhibits superior statistical properties compared to dominant baselines. It maintains a healthy long-tail vocabulary distribution (Zipf's law) and an optimal syntactic density peaking at four words. Crucially, the dataset dismantles the ``noun monopoly'' prevalent in standard corpora by elevating the proportion of adjectives to 22\%. This dense population of visual modifiers supplies the critical high-entropy variations required to prevent representation collapse during large-scale unified grounding.

\section{Theoretical Analysis: Representation Diversity}
\label{sec:theory}

We characterize representation diversity by which feature directions
receive sustained alignment supervision (Fig.~\ref{fig:capacity}).
Let $x_m\in\mathbb{R}^C$ be the shared visual token at spatial location
$m$ (after the final fusion layer) and $k$ a unit text (key) direction.
Since the attention logit $a_m=x_m^\top k$ is linear in $x_m$, the
discriminative signal carried by $k$ is the spatial variance of the
logits, defining the \emph{directional alignment capacity}
\begin{equation}
\mathrm{cap}(k)
:=
\mathrm{Var}_m(x_m^\top k)
=
k^\top\Xi_Xk,
\end{equation}
with $\Xi_X$ the covariance of $\{x_m\}$; directions with
$\mathrm{cap}(k)=0$ form the \emph{alignment-blind subspace}.
Appendix~B shows that under weight decay only gradient-sustained
directions retain non-zero capacity, so the diversity of an objective is
governed by the subspace its gradients span.

With $N_c$ expressions per image, $N$ total text tokens, and feature
dimension $C$, the three objectives span gradient subspaces of dimension
\begin{equation}
\label{eq:ladder}
\underbrace{N_c}_{\text{Contrastive}}
\;<\;
\underbrace{N-N_c}_{\texttt{mACH}}
\;<\;
\underbrace{C}_{\text{mACH+JEPA}},
\end{equation}
as upper bounds (requiring $N_c<N/2$ and $N-N_c<C$; both hold in our
data): contrastive learning supervises only the pooled-expression
subspace; \texttt{mACH} expands to the centered token subspace (softmax
invariance removes one common-mode direction per expression); and the
auxiliary JEPA objective provides gradient support in almost every
direction. Thus, only the dual-stream objective is
almost surely free of alignment-blind directions and preserves
representation diversity throughout the feature space. Formal statements
and proofs are in Supp.~Mat.~B.

\section{Experiments}
\label{sec:experiments}
\subsection{Datasets and Evaluation Metrics}

We evaluate our method on three widely used referring expression comprehension benchmarks: RefCOCO/+/g. To ensure a rigorous and fair comparison, all evaluated models are strictly benchmarked on a cleaned version of these datasets~\cite{chen2025revisiting}, which rectifies original annotation noise and spatial ambiguities. Following the standard evaluation protocol, a prediction is regarded as correct if its Intersection over Union (IoU) with the ground-truth bounding box exceeds 0.5, and Top-1 accuracy (\%) is reported as the evaluation metric.

\subsection{Implementation Details}
\label{subsec:implementation_details}

We implement both CNN-based and DETR-based variants within a unified PyTorch framework. Unless otherwise specified, a DINOv3 ConvNeXt-Tiny \cite{simeoni2025dinov3} is adopted as the visual backbone. The CNN variant employs an FPN neck with the proposed mACH head, while the DETR variant(detailed in Supp. Mat.~A) follows the RT-DETR architecture with text embeddings injected into the AIFI module for early vision-language interaction. To support multilingual grounding, we use the frozen \texttt{Qwen3-VL-Embedding-2B}~\cite{li2026qwen3} as the language encoder. All models are trained using AdamW with standard data augmentation. Detailed architectural configurations, optimization settings, and hyperparameters are provided in Supp. Mat.~E.

\begin{figure*}[t]
   \centering
   \begin{subfigure}[t]{0.245\textwidth}
       \centering
       \includegraphics[width=0.98\linewidth]{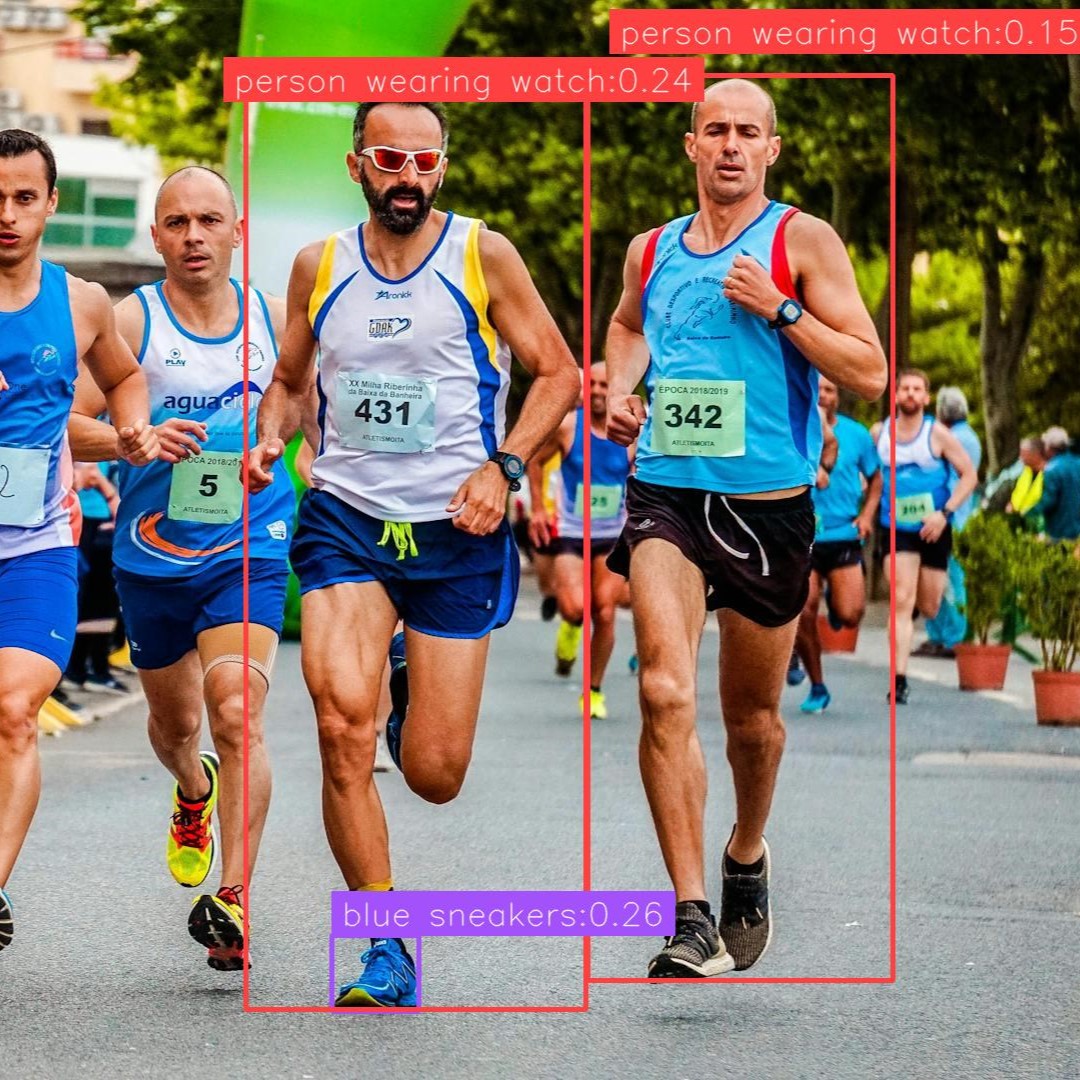}
       \caption{\textcolor{red}{person wearing watch}. \textcolor{violet}{blue sneakers}.}
       \label{fig:runningperson}
   \end{subfigure}
   \begin{subfigure}[t]{0.245\textwidth}
       \centering
       \includegraphics[width=0.98\linewidth]{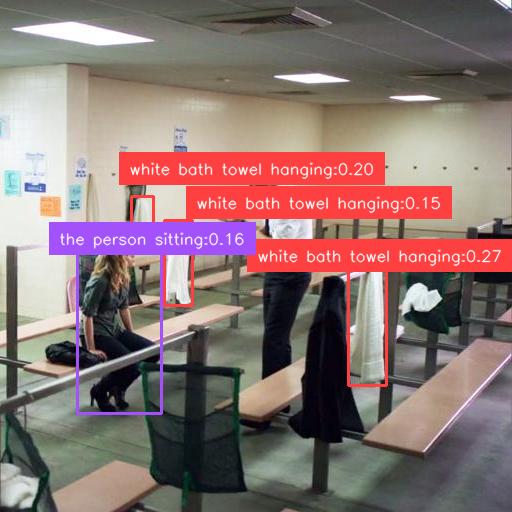}
       \caption{\textcolor{violet}{the person sitting}. \textcolor{red}{white bath towel hanging}.}
       \label{fig:obj365_22}
   \end{subfigure}%
   \begin{subfigure}[t]{0.245\textwidth}
       \centering
       \includegraphics[width=0.98\linewidth]{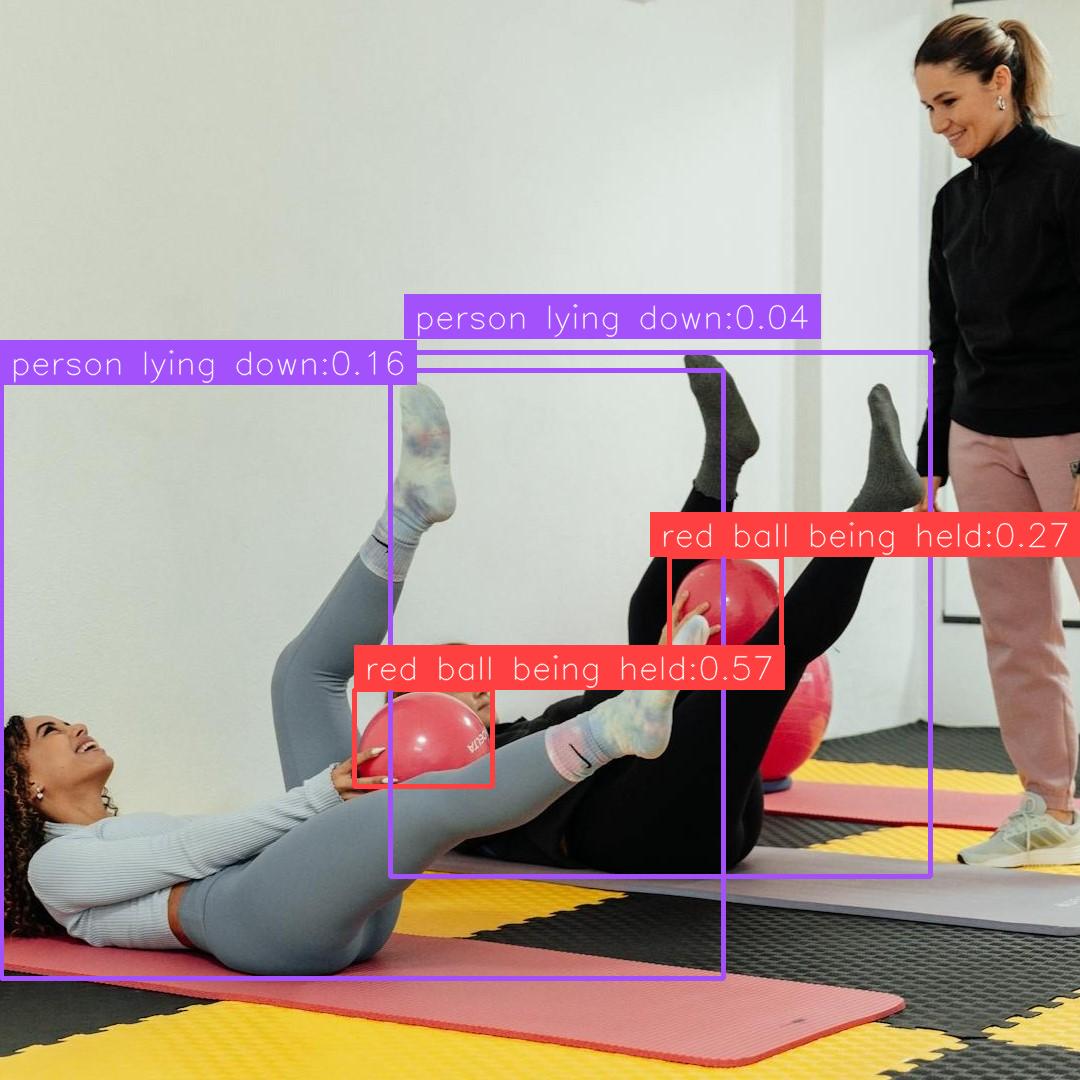}
       \caption{\textcolor{violet}{person lying down}. \textcolor{red}{red ball being held}.}
       \label{fig:ballinlegs}
   \end{subfigure}
   \begin{subfigure}[t]{0.245\textwidth}
       \centering
       \includegraphics[width=0.98\linewidth]{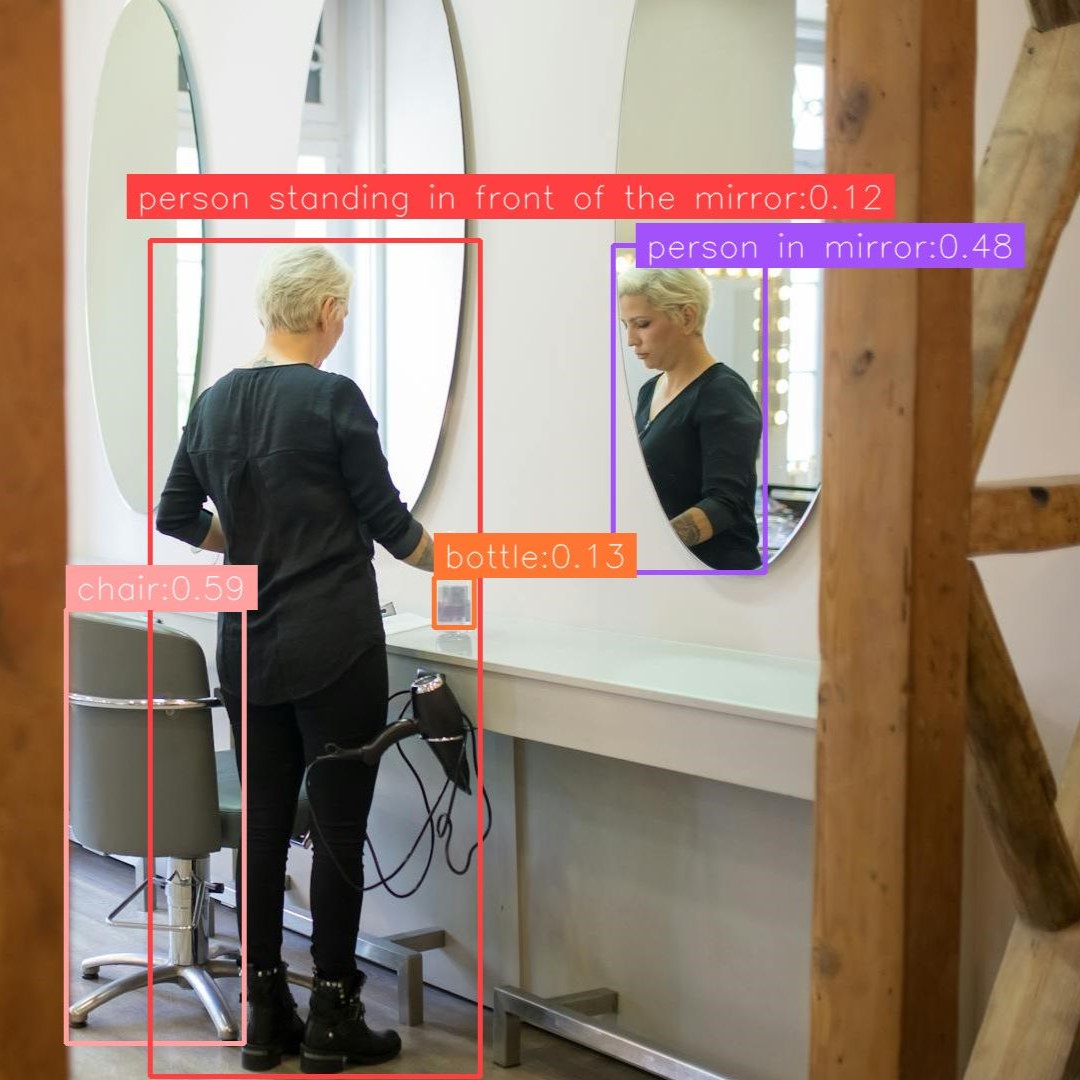}
       \caption{\textcolor{violet}{person in mirror}. \textcolor{red}{person standing in front of the mirror}. \textcolor{pink}{chair}. }
       \label{fig:personinmirror}
   \end{subfigure}%
   \caption{\textbf{Qualitative grounding results.} Text queries and corresponding predicted bounding boxes are rendered in matching colors, demonstrating robust multi-target localization under complex natural language descriptions. Zoom in for best view.}
   \label{fig:resultsvisual}
\end{figure*}
\subsection{Main Results} 
 
We evaluate our framework on RefCOCO, RefCOCO+, and RefCOCOg under both zero-shot (unified generalist) and benchmark-specific fine-tuning protocols (Tab.~\ref{tab:refcoco}). 
 
\noindent \textbf{Unified Generalist Evaluation.} 
Without target fine-tuning, our compact 75M model demonstrates strong zero-shot generalization across all three benchmarks. On RefCOCO, it achieves \textbf{85.3/89.0/82.5} (val/testA/testB), outperforming the 172M GDINO-T by \textbf{+11.3/+14.1/+23.2\%} and matching heavy MLLMs (e.g., 13B LISA++-L2 and 7B GSVA) using only 0.6\% of LISA's parameters and a lower resolution ($640^2$ vs. $1024^2$). This strong capability extends to attribute-rich RefCOCO+ (\textbf{71.8/78.2/62.7}) and long-expression RefCOCOg (\textbf{76.3/75.8}), consistently outperforming lightweight OVD baselines (e.g., YOLOE, ExpAlign, GLIP-T) by substantial margins. We attribute this advantage to our \texttt{O365-Caption} pre-training and JEPA regularization, which preserve fine-grained spatial-linguistic correspondence rather than disrupting it with aggressive geometric augmentations common in standard OVD recipes. 
 
\noindent \textbf{Benchmark-Specific Fine-Tuning.} 
When fine-tuned on target datasets, our framework achieves state-of-the-art performance across all three benchmarks. On RefCOCO, our model reaches \textbf{91.7/93.0/90.2}, consistently outperforming fine-tuned GDINO-T (89.2/91.9/86.0) and PropVG (89.0/91.6/85.7). On RefCOCO+, it achieves \textbf{76.9\%} accuracy on testB. Furthermore, on RefCOCOg, our model achieves \textbf{85.1/86.0} (val/test), outperforming PropVG (83.5/84.4) despite using significantly fewer parameters (75M vs. 490M) and a lower input resolution ($640^2$ vs. $800\!\times\!1333$). These gains confirm that our unified pre-training serves as a robust initialization for downstream REC.

Qualitative visualizations in Fig.~\ref{fig:resultsvisual} further validate the superior localization capability of our unified framework.

\subsection{Ablation Studies} 
\label{subsec:ablation_studies} 
 
We perform ablation studies under the zero-shot generalist setting using ConvNeXt-Tiny as the default backbone unless otherwise specified. Tab.~\ref{tab:unified_ablations} systematically deconstructs the contributions of our head design, pre-training data, auxiliary loss weighting, and architectural choices. 


\begin{table}[!t]
\centering
\caption{\textbf{Ablation studies on core framework components.} Evaluated under the unified setting across RefCOCO/+/g. Bold indicates the default configuration.}
\label{tab:unified_ablations}
\small 
\setlength{\tabcolsep}{1.5pt}
\begin{tabular}{l|ccc|ccc|cc}
\toprule
\multirow{2}{*}{\textbf{Configuration}} & \multicolumn{3}{c}{\textbf{RefCOCO}} & \multicolumn{3}{c}{\textbf{RefCOCO+}} & \multicolumn{2}{c}{\textbf{RefCOCOg}} \\
& val & testA & testB & val & testA & testB & val & test \\
\midrule
\multicolumn{9}{l}{\textit{(a) Head Components (Fixed Data: GoldG-f, O365-C)}} \\
\midrule
Contrastive & 78.6 & 80.3 & 73.8 & 61.0 & 66.4 & 53.2 & 72.7 & 72.3 \\
\texttt{mACH} & 83.9 & 87.5 & 79.3 & 68.5 & 77.0 & 59.4 & 74.8 & 75.0 \\
\textbf{\texttt{mACH} + \texttt{JEPA}} & \textbf{85.3} & \textbf{89.0} & \textbf{82.5} & \textbf{71.8} & \textbf{78.2} & \textbf{62.7} & \textbf{76.3} & \textbf{75.8} \\
\midrule
\multicolumn{9}{l}{\textit{(b) Pre-train Data (Fixed Head Architecture: full pipeline)}} \\
\midrule
GoldG, O365 & 77.0 & 69.3 & 71.7 & 62.7 & 69.2 & 56.0 & 66.6 & 65.7 \\
GoldG, O365-C & 85.3 & 88.1 & 82.1 & 71.7 & 77.9 & 61.4 & 76.2 & 75.7\\
GoldG-f, O365-C & \textbf{85.3} & \textbf{89.0} & \textbf{82.5} & \textbf{71.8} & \textbf{78.2} & \textbf{62.7} & \textbf{76.3} & \textbf{75.8} \\
\midrule
\multicolumn{9}{l}{\textit{(c) Effect of Auxiliary Jepa Weight $\alpha$ }} \\
\midrule
$\alpha = 0.0$ & 83.9 & 87.5 & 79.3 & 68.5 & 77.0 & 59.4 & 74.8 & 75.0 \\
$\alpha = 0.05$ & 84.8 & 88.5 & 81.4 & 70.8 & 77.9 & 61.5 & 75.9 & 75.5 \\
$\alpha = \mathbf{0.1}$  & \textbf{85.3} & \textbf{89.0} & \textbf{82.5} & \textbf{71.8} & \textbf{78.2} & \textbf{62.7} & \textbf{76.3} & \textbf{75.8} \\
$\alpha = 0.2$ & 84.7 & 88.3 & 81.1 & 70.5 & 77.6 & 61.9 & 75.4 & 75.2 \\
\midrule
\multicolumn{9}{l}{\textit{(d) Architectural Components}} \\
\midrule
\textbf{CNN}(30 ep) & 85.3 & 89.0 & 82.5 & 71.8 & 78.2 & 62.7 & 76.3 & 75.8 \\
RT-DETR(30 ep) & 84.0 & 87.8 & 80.2 & 69.9 & 77.5 & 62.7 & 75.4 & 75.2 \\
RT-DETR(90 ep) & \textbf{86.8} & \textbf{90.0} & \textbf{83.5} & \textbf{74.5} & \textbf{82.4} & \textbf{64.1} & \textbf{78.6} & \textbf{77.9} \\
\bottomrule
\end{tabular}
\end{table}

\noindent \textbf{Head Components (Tab.~\ref{tab:unified_ablations}a).} Replacing the basic contrastive head with \texttt{mACH} boosts RefCOCO val accuracy by \textbf{+5.3\%} (78.6\% vs. 83.9\%), validating the necessity of dense token-level cross-modal interaction. Integrating the auxiliary \texttt{JEPA} stream further lifts performance across all splits (reaching 85.3\% on RefCOCO val and 82.5\% on testB), demonstrating that auxiliary representation learning provides effective regularization for grounding. 
 
\noindent \textbf{Linguistic Supervision (Tab.~\ref{tab:unified_ablations}b).} Converting Objects365 tags into context-rich referring expressions yields a substantial \textbf{+8.3\%} improvement on RefCOCO val (77.0\% vs. 85.3\%) using identical images and bounding boxes, confirming that textual richness is vital for unified pre-training. Further annotation cleaning (\texttt{GoldG-f}) provides additional gains on challenging test splits (e.g., +0.9\% on RefCOCO testA). 
 
\noindent \textbf{Auxiliary JEPA Weight (Tab.~\ref{tab:unified_ablations}c).} Setting the loss weight to $\alpha = 0.1$ achieves an optimal balance between discriminative grounding and generative reconstruction, whereas higher weights ($\alpha = 0.2$) slightly diminish performance due to over-regularization. 
 
\noindent \textbf{Architecture \& Convergence (Tab.~\ref{tab:unified_ablations}d).} While our CNN-based default converges rapidly within 30 epochs (85.3\% on RefCOCO val), training an RT-DETR architecture for 90 epochs reaches the global peak performance of \textbf{86.8/90.0/83.5} on RefCOCO, demonstrating strong architectural flexibility and scaling potential for our pipeline. 
 
\noindent \textbf{Scalability Across Backbone Sizes.} 
We evaluate framework scalability across ConvNeXt Tiny, Small, and Base variants in Tab.~\ref{tab:ablation_scaling_backbone}. Performance improves steadily as visual capacity and head feature dimension $C$ expand. Notably, while scaling backbone depth from Tiny to Small ($C=768$) yields moderate gains, widening the cross-modal interaction channel to $C=1024$ in ConvNeXt-Base produces a significant jump, boosting RefCOCO+ testA by \textbf{+2.9\%} and RefCOCOg test by \textbf{+1.4\%} over Tiny. This confirms expanding the head capacity $C$ effectively eliminates representation bottlenecks, allowing our \texttt{mACH} and \texttt{JEPA} modules to scale seamlessly with larger feature spaces without saturation. 
 
\noindent \textbf{Inference Efficiency.} 
We evaluate framework latency and peak GPU memory across varying numbers of text queries ($N$) in Tab.~\ref{tab:efficiency}. Measurements are conducted on a single NVIDIA RTX 4090 at batch size 1.
\begin{table}[t]
\centering
\caption{\textbf{Scalability across visual backbone sizes.} Zero-shot performance evaluated across ConvNeXt scales. }
\label{tab:ablation_scaling_backbone}
\small 
\setlength{\tabcolsep}{1.3pt}
\begin{tabular}{l|ccc|ccc|cc}
\toprule
\multirow{2}{*}{\textbf{Backbone Scale}} & \multicolumn{3}{c}{\textbf{RefCOCO}} & \multicolumn{3}{c}{\textbf{RefCOCO+}} & \multicolumn{2}{c}{\textbf{RefCOCOg}} \\
& val & testA & testB & val & testA & testB & val & test \\
\midrule
Tiny ($C\!=\!768$)           & 85.3 & 89.0 & 82.5 & 71.8 & 78.2 & 62.7 & 76.3 & 75.8 \\
Small ($C\!=\!768$)            & 86.1 & 89.1 & 82.6 & 72.4 & 79.3 & 63.3 & 76.8 & 76.1 \\
Base ($C\!=\!1024$)             & \textbf{86.3} & \textbf{89.8} & \textbf{83.8} & \textbf{73.7} & \textbf{81.1} & \textbf{64.4} & \textbf{76.8} & \textbf{77.2} \\
\bottomrule
\end{tabular}
\end{table}
\begin{table}[t] 
 \centering 
 \caption{\textbf{Inference efficiency comparison.} Results are reported as triplets corresponding to $N=1/5/10$, where $N$ denotes the number of text queries (average 5 tokens/query) evaluated per image.} 
 \label{tab:efficiency} 
 \small 
 \setlength{\tabcolsep}{5pt} 
 \begin{tabular}{l|c|c} 
 \toprule 
 \textbf{Method}  & \textbf{Latency (ms)} $\downarrow$ & \textbf{Peak Mem.} (GB) $\downarrow$ \\ 
 \midrule 
 Contrastive    & 21/23/24 & 1.06/1.08/1.10 \\ 
 \texttt{mACH}  & 26/26/27 & 1.13/1.26/1.35 \\ 
 \bottomrule 
 \end{tabular} 
 \end{table}

\subsection{Spectral Evidence}
\label{sec:spectral_evidence}

To examine how different objectives shape the learned representation,
we collect the shared visual token features after the final fusion layer
from all RefCOCOg images and estimate their empirical feature
covariance
$
\Xi_X
=
\frac{1}{M}
\sum_{i=1}^{M}
(x_i-\bar{x})(x_i-\bar{x})^\top ,
$
where $M$ is the total number of visual tokens.
According to our theory, the eigenspectrum of $\Xi_X$
measures the distribution of directional alignment capacity:
discriminative objectives preserve capacity only inside the
language-conditioned subspace, whereas the additional JEPA objective
maintains capacity in the remaining directions.
\begin{figure}[!t] 
\centering 
\includegraphics[width=1\linewidth]{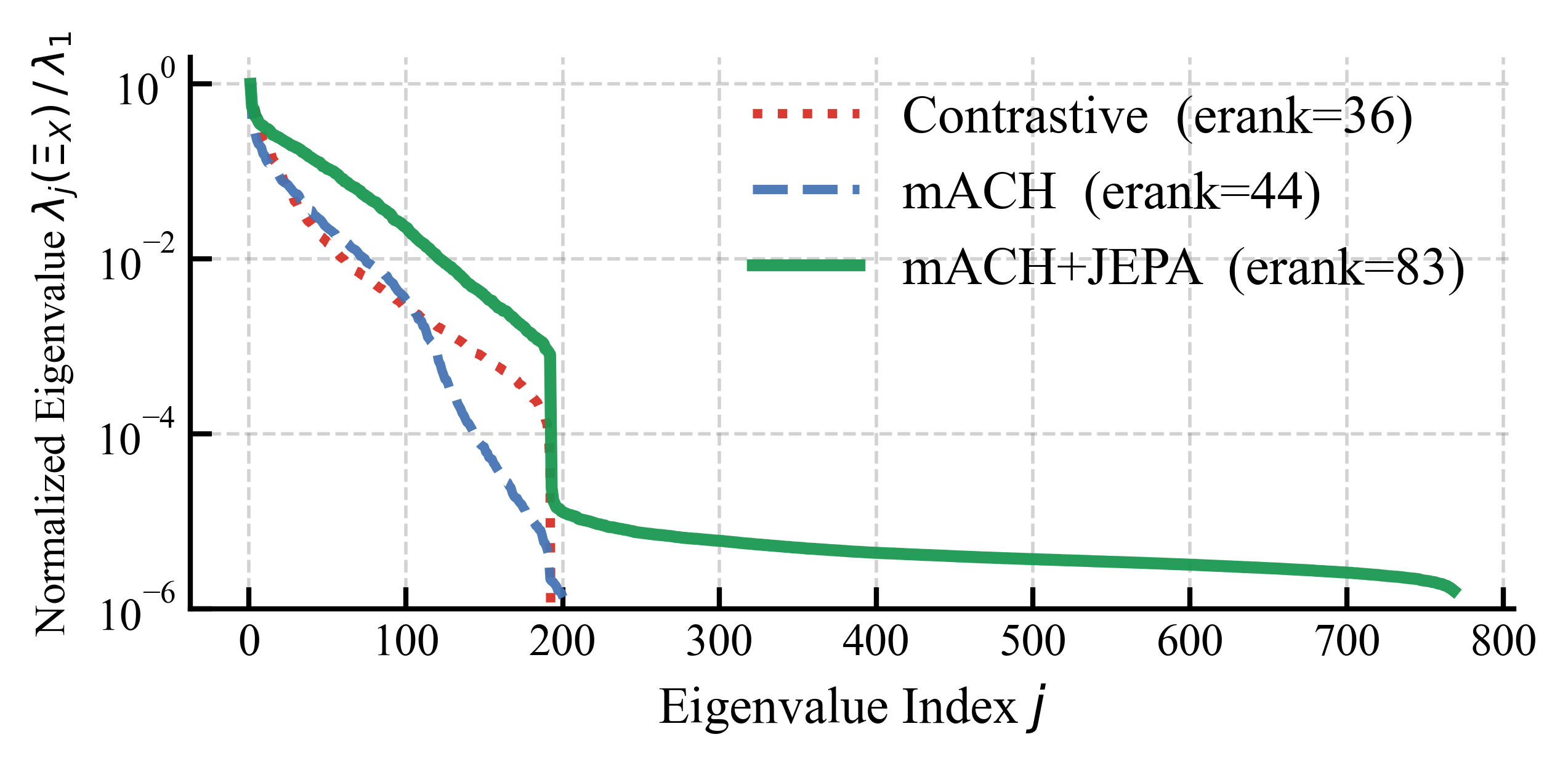} 
\caption{ 
 \textbf{Alignment-capacity spectrum.} 
 Normalized eigenspectra of the empirical feature covariance $\Xi_X$ 
 computed from shared visual token features after the final fusion 
 layer on RefCOCOg-val. Contrastive learning and mACH exhibit a shared spectral cliff around $j\!\approx\!200$, whereas mACH+JEPA 
 retains a non-vanishing spectral tail across all $768$ dimensions, consistent with the predicted spectral floor (eigenvalues below ${\sim}10^{-6}$ are at the numerical floor; cliff 
 positions, not depths, are interpretable). Reported 
 effective ranks summarize the spectrum and indicate 
 progressively richer representation diversity. 
 }
\label{fig:m1_spectrum} 
\end{figure}

Figure~\ref{fig:m1_spectrum} confirms this prediction. Contrastive learning and mACH exhibit an identical spectral cliff around $j\!\approx\!200$, after which the eigenvalues fall below the numerical precision of float32 accumulation (${\lesssim}10^{-6}\lambda_1$), indicating feature variation is confined to a language-conditioned subspace. In contrast, mACH+JEPA preserves a non-vanishing spectral tail (${\sim}10^{-5}$), consistent with the positive spectral floor predicted by our theory, and maintains measurable variance throughout the entire ambient space ($C=768$).

To summarize the spectrum with a single statistic, we report the
\emph{effective rank}: $\mathrm{erank}(\Xi_X) = \exp(-\sum_j p_j\log p_j)$, $p_j=\frac{\lambda_j}{\sum_k\lambda_k}$
which measures how many feature directions carry comparable variance
\cite{garrido2023rankme}. The effective rank increases
monotonically from 36 (Contrastive) to 44 (mACH) and 83 (mACH+JEPA),
consistent with the predicted dimension ladder
$N_c < N\!-\!N_c < C$, and indicating progressively richer representation
diversity under the dual-stream objective.

\section{Conclusion}
\label{sec:conclusion}

We revisited referring expression comprehension from the perspective of unified open-vocabulary grounding and presented a holistic data-model co-design that jointly improves linguistic supervision and visual representation learning. Specifically, our framework combines a lightweight broadcast-based grounding head with an inference-free JEPA auxiliary objective to preserve representation diversity, while \texttt{O365-Caption} enriches grounding supervision through diverse natural language descriptions. Extensive experiments demonstrate strong cross-dataset generalization and competitive performance using a static checkpoint. Beyond the proposed framework, our theoretical and empirical results suggest preserving representation diversity is an important design principle for scaling unified vision-language grounding models beyond benchmark-specific specialization.

\bibliography{aaai2027}


\clearpage 
\onecolumn

\setcounter{table}{0}  
\setcounter{figure}{0}

\setcounter{equation}{0}
\renewcommand{\thetable}{A\arabic{table}}
\renewcommand{\thefigure}{A\arabic{figure}}

\renewcommand{\theequation}{A\arabic{equation}}
\section*{Appendix A: Generalization to Transformer-based Grounding Frameworks}
\label{app:detr}

Although the main paper presents the CNN-based implementation for clarity, the proposed Modulated Attention-Contrastive Head (mACH) is architecture-agnostic and can be seamlessly integrated into transformer-based grounding detectors.

\begin{figure}[!htbp]
    \centering
    \includegraphics[width=1\linewidth]{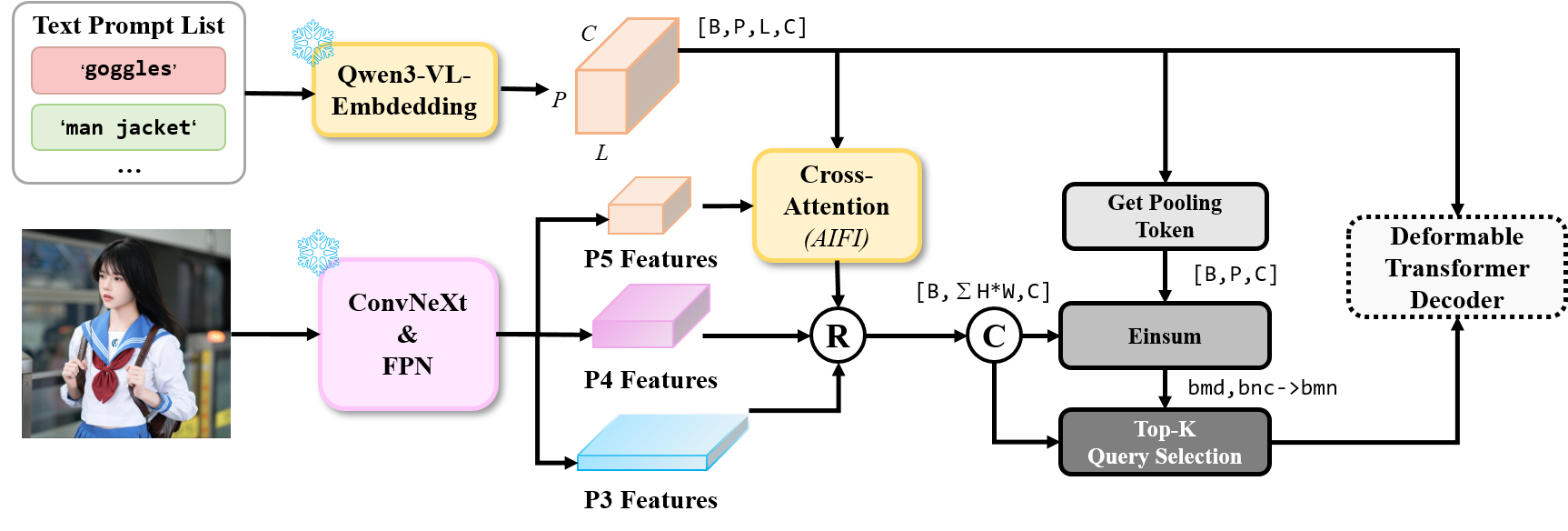}
    \caption{Transformer-based implementation of the proposed mACH framework. Early multimodal interaction is introduced through the AIFI module, while the conventional grounding head is replaced by mACH. The deformable transformer decoder and box regression branch remain identical to the original detector.}
    \label{fig:detr_architecture}
\end{figure}

Figure~\ref{fig:detr_architecture} illustrates our transformer-based implementation using a deformable DETR-style framework. Two lightweight modifications are introduced while preserving the original detection pipeline. First, to enable early vision-language interaction, the finest-scale visual features are concatenated with the language embeddings and jointly processed by the AIFI module. Second, the conventional grounding head is replaced by the proposed mACH head, while the deformable transformer decoder and the bounding-box regression branch remain unchanged.

Figure~\ref{fig:detr_decoder_compare} further compares our grounding framework with Grounding DINO. Unlike Grounding DINO, which injects language cross-attention into every decoder layer, our design keeps the deformable decoder purely visual and performs cross-modal interaction only in the lightweight mACH head after object decoding. This decouples object decoding from language grounding, preserves the standard box-level Hungarian assignment during training, and avoids the token-aware matching strategy adopted by Grounding DINO, resulting in a substantially simpler training pipeline.

\begin{figure}[!t]
    \centering
    \includegraphics[width=1\linewidth]{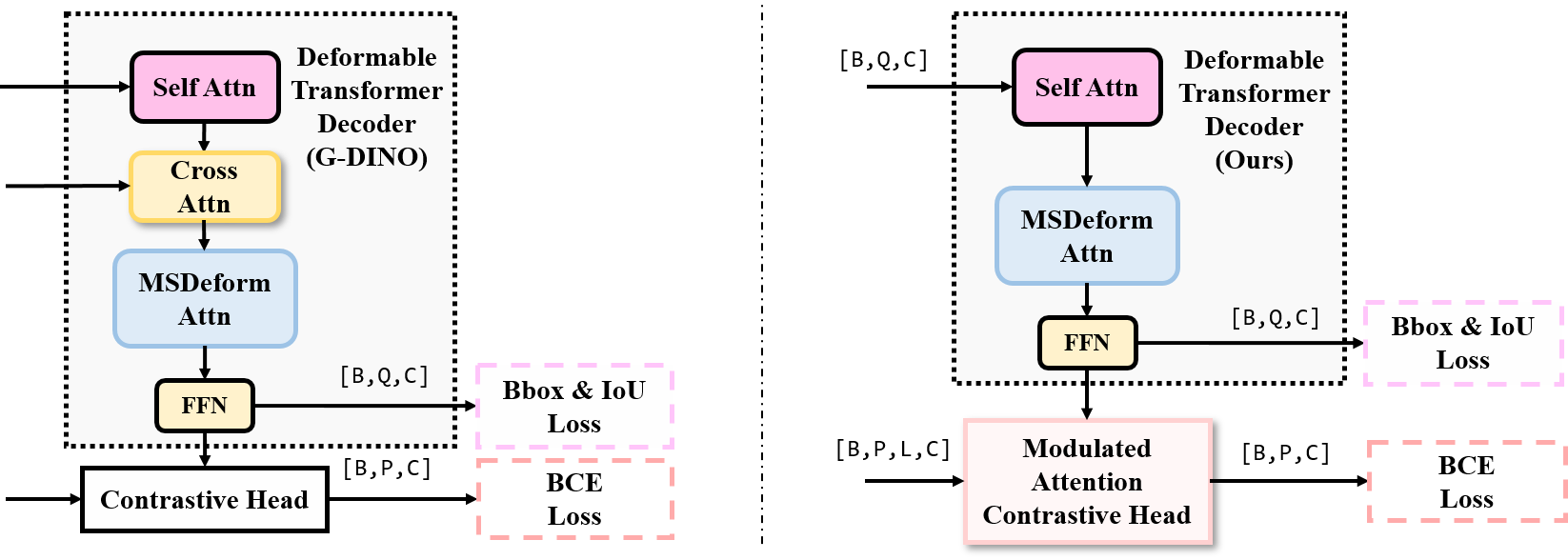}
    \caption{Comparison between Grounding DINO and the proposed transformer-based grounding framework. Grounding DINO performs language cross-attention inside every decoder layer, whereas our approach keeps the decoder purely visual and delegates cross-modal interaction to the lightweight mACH prediction head. Consequently, the proposed design preserves the standard box-level Hungarian assignment and simplifies the overall training pipeline.}
    \label{fig:detr_decoder_compare}
\end{figure}
As demonstrated by the experimental results in the main paper, this design consistently improves both CNN-based and transformer-based grounding architectures, indicating that mACH serves as a generic grounding head independent of the underlying visual backbone.

\section*{Appendix B: Directional Alignment Capacity}
\label{sec:appendix_gradient_rank}

This appendix formalizes the argument of the Theoretical Analysis section: each architecture determines which gradient directions are reachable on the shared features, and the reachable dimensions of the three objectives form the ladder $N_c<N-N_c<C$. We define a directional measure of localization signal (Definition~\ref{def:cap}), show that it persists only on gradient-sustained directions (Lemma~\ref{lem:decay}), compute the reachable subspace of each objective (Lemmas~\ref{lem:contrast}--\ref{lem:jepa_full}), and compare the three paradigms (Theorem~\ref{thm:comparison}, Corollaries~\ref{cor:diversity} and \ref{cor:ood}).

\subsection{Setup}
\label{subsec:setup}

\textbf{Notation.} Each image is represented by visual tokens $X=[x_1,\dots,x_M]\in\mathbb{R}^{C\times M}$ (the shared trunk output; the query projection is absorbed into $X$), paired with $N_c$ referring expressions, the $i$-th given by token embeddings $T^{(i)}=[t_1^{(i)},\dots,t_{L_i}^{(i)}]\in\mathbb{R}^{d\times L_i}$; write $N:=\sum_{i=1}^{N_c}L_i$. Positive expressions are fixed per image; negatives are resampled every iteration. Let $\Xi_X:=\mathbb{E}[(x_m-\bar{x})(x_m-\bar{x})^\top]$ be the centered token covariance, and for any objective $\mathcal{L}$ acting on $X$, define its gradient field and gradient second moment
\begin{equation}
\label{eq:grad_field}
G(\mathcal{L}):=\nabla_X\mathcal{L},
\qquad
\Gamma(\mathcal{L}):=\mathbb{E}\big[G(\mathcal{L})\,G(\mathcal{L})^\top\big].
\end{equation}
The head broadcasts $X$ along the text axis; backpropagation through a broadcast is branch summation, so all $N_c$ text branches accumulate their gradients on the \emph{same} $X$:
\begin{equation}
\label{eq:broadcast_adjoint}
G(\mathcal{L}_{\mathrm{m}})=\frac{1}{N_c}\sum_{i=1}^{N_c}G^{(i)} .
\end{equation}

\textbf{\texttt{mACH} forward map.} Keys $k_l^{(i)}=W_Kt_l^{(i)}$, logits $a_{m,l}^{(i)}=x_m^\top k_l^{(i)}/\sqrt{C}$, attention $A_{m,l}^{(i)}=\mathrm{softmax}_l(a_{m,l}^{(i)})$, and scores
\begin{equation}
\label{eq:mach_forward}
S_m^{(i)}=\sum_{l=1}^{L_i}A_{m,l}^{(i)}\,c_l^{(i)}+b,
\qquad
c_l^{(i)}:=e^{\tau}w_\psi^\top W_Vt_l^{(i)}\in\mathbb{R}.
\end{equation}
The logit scale $e^\tau$ and bias $b$ shift and rescale scores uniformly and play no directional role. The grounding objective is $\mathcal{L}_{\mathrm{m}}=\mathcal{L}_{\mathrm{BCE}}(S)$.

\textbf{JEPA stream.} A student $\mathcal{P}_\theta:\mathbb{R}^C\to\mathbb{R}^C$ is paired with an EMA teacher updated as
\begin{equation}
\label{eq:ema}
\mathcal{P}_{\mathrm{EMA}}
\leftarrow
\lambda\,\mathcal{P}_{\mathrm{EMA}}+(1-\lambda)\,\mathcal{P}_\theta,
\qquad \lambda\in(0,1).
\end{equation}
For $z_m=\mathcal{P}_\theta(x_m)$ and a stochastically resampled mask $\Omega$,
\begin{equation}
\label{eq:mask}
z_m^{\mathrm{msk}}=
\begin{cases}
\mu+\epsilon, & m\in\Omega,\\
z_m, & m\notin\Omega,
\end{cases}
\qquad \epsilon\sim\mathcal{N}(0,\sigma^2 I)
\quad\Longrightarrow\quad
\frac{\partial z_m^{\mathrm{msk}}}{\partial x_m}=\mathbf{0},\ \ m\in\Omega .
\end{equation}
With $T'=W_{\mathrm{proj}}[T^{(1)},\dots,T^{(N_c)}]\in\mathbb{R}^{C\times N}$, a joint self-attention predictor (parameters $\phi$) operates on $[(Z^{\mathrm{msk}})^\top;(T')^\top]$ with block affinity
\begin{equation}
\label{eq:block_attn}
\mathcal{A}=
\begin{pmatrix}
\mathcal{A}_{vv} & \mathcal{A}_{vt}\\
\mathcal{A}_{tv} & \mathcal{A}_{tt}
\end{pmatrix},
\qquad
\hat{z}_{\mathrm{pred},o}
=\sum_{j\notin\Omega}\mathcal{A}_{vv,\,o,j}\,W_V^{(v)}z_j
+\sum_{l=1}^{N}\mathcal{A}_{vt,\,o,l}\,W_V^{(t)}t'_l,
\quad o\in\Omega ,
\end{equation}
where constant contributions of masked-token values are absorbed into the predictor bias. With $\bar{u}:=u/\|u\|$ and $z_{\mathrm{tgt},o}=\mathcal{P}_{\mathrm{EMA}}(x_o)$,
\begin{equation}
\label{eq:jepa_loss}
\mathcal{L}_{\mathrm{J}}
=\frac{1}{|\Omega|}\sum_{o\in\Omega}
\Big[\big(1-\langle\bar{\hat{z}}_{\mathrm{pred},o},\bar{z}_{\mathrm{tgt},o}\rangle\big)
+\beta\,\mathrm{SmoothL1}\big(\bar{\hat{z}}_{\mathrm{pred},o},\bar{z}_{\mathrm{tgt},o}\big)\Big],
\qquad
\mathcal{L}_{\mathrm{tot}}=\mathcal{L}_{\mathrm{m}}+\alpha\,\mathcal{L}_{\mathrm{J}} .
\end{equation}

\textbf{Roles of the four blocks.} The loss supervises only masked \emph{visual} predictions. Gradient therefore flows back to $X$ through $\mathcal{A}_{vv}$, which routes the driving signal from masked positions to unmasked visual tokens. The $\mathcal{A}_{vt}$ path lets text context modulate the predictions -- and couples the affinities to $T$ through the softmax normalization -- but the results below do not rely on it. The text-side outputs, governed by $\mathcal{A}_{tv}$ and $\mathcal{A}_{tt}$, receive no supervision and carry no gradient.

Grounding requires the score map to vary across positions: a flat score map carries no localization signal. Since the attention logits are linear in the visual tokens, $a_m=x_m^\top k$, their spatial discriminability is exactly $\mathrm{Var}_m(a_m)$; the following definition is therefore a property of the attention mechanism itself rather than an ad hoc construct.

\begin{definition}[Directional Alignment Capacity]
\label{def:cap}
For a unit direction $k\in\mathbb{R}^{C}$,
\begin{equation}
\mathrm{cap}(k)\;:=\;\mathrm{Var}_m\big(x_m^\top k\big)\;=\;k^\top\Xi_X\,k .
\end{equation}
$\mathrm{cap}(k)$ is the spatial variance of the attention logits available to any token whose key points along $k$; $\mathrm{cap}(k)=0$ implies position-independent attention and hence a flat score map. We call $\{k:\mathrm{cap}(k)=0\}$ the \emph{alignment-blind subspace} and its complement the \emph{alignment-active subspace}.
\end{definition}

\subsection{Gradient-Sustained Capacity}
\label{subsec:dynamics}

Alignment capacity is not a static property of the representation: under weight decay, feature variance contracts unless continuously replenished by the optimization signal. We model this as $\dot{X}=-GB-\lambda X$, where $\lambda>0$ is the decay rate and $B\succeq0$ an arbitrary backbone-induced preconditioner (exact for a terminal linear map $X=WH$ with decayed $W$, where $B=H^\top H$; and for degree-$1$ homogeneous feature maps with full decay, where $B=JJ^\top$). Only one property is used below: the excitation term vanishes on directions orthogonal to all gradient columns.

\begin{lemma}[Gradient sustains directional variance]
\label{lem:decay}
If $k^\top G(\mathcal{L})=\mathbf{0}$ almost surely along the trajectory, then
\begin{equation}
k^\top\Xi_X^{(t)}k\;\le\;e^{-2\lambda t}\,k^\top\Xi_X^{(0)}k,
\end{equation}
i.e.\ $\mathrm{cap}(k)\to0$: directional variance persists only where the gradient signal sustains it.
\end{lemma}

\begin{proof}
For $e_m:=(k^\top x_m)^2$, $\dot{e}_m=-2(k^\top x_m)\,k^\top[GB]_m-2\lambda e_m$. Since $[GB]_m=\sum_jg_jB_{jm}$ and $k^\top g_j=0$ a.s.\ for every column $j$, the first term vanishes for any $B\succeq0$, giving $\dot{e}_m=-2\lambda e_m$; the token mean decays at the same rate, yielding the claim for the centered variance.
\end{proof}

\subsection{Gradient Subspaces of the Three Objectives}
\label{subsec:subspaces}

\begin{lemma}[Pooled contrastive head]
\label{lem:contrast}
For $S_m^{(i)}=(t^{(i)})^\top x_m$ with pooled text vector $t^{(i)}\in\mathbb{R}^C$, writing the BCE residual $\delta_m^{(i)}:=\partial\mathcal{L}_{\mathrm{BCE}}/\partial S_m^{(i)}$,
\begin{equation}
G_{\mathrm{contrast}}=\frac{1}{N_c}\sum_{i=1}^{N_c}t^{(i)}(\boldsymbol{\delta}^{(i)})^\top,
\qquad
\mathrm{Col}(G_{\mathrm{contrast}})\subseteq\mathrm{span}\{t^{(i)}\}_{i=1}^{N_c},
\quad \dim\le N_c .
\end{equation}
\end{lemma}

\begin{proof}
Since $\nabla_{x_m}S_m^{(i)}=t^{(i)}$, the branch gradient at token $m$
is $g_m^{(i)}=\delta_m^{(i)}t^{(i)}$; stacking the $M$ columns,
$G^{(i)}=t^{(i)}\big(\boldsymbol{\delta}^{(i)}\big)^\top$ with
$\boldsymbol{\delta}^{(i)}:=(\delta_1^{(i)},\dots,\delta_M^{(i)})^\top$.
Assembling via \eqref{eq:broadcast_adjoint},
\begin{equation}
G_{\mathrm{contrast}}
=\frac{1}{N_c}\sum_{i=1}^{N_c}t^{(i)}\big(\boldsymbol{\delta}^{(i)}\big)^\top
\;\Longrightarrow\;
\mathrm{Col}(G_{\mathrm{contrast}})
\subseteq\sum_{i=1}^{N_c}\mathrm{span}\big\{t^{(i)}\big\}
=\mathrm{span}\big\{t^{(i)}\big\}_{i=1}^{N_c},
\qquad \dim\le N_c .
\end{equation}
\end{proof}

\begin{lemma}[\texttt{mACH} head]
\label{lem:mach}
With the \emph{centered key subspace}
\begin{equation}
\label{eq:centered_key}
\mathcal{V}_{\mathcal{T}}^{0}
:=\sum_{i=1}^{N_c}\big\{W_KT^{(i)}\beta:\mathbf{1}^{\top}\beta=0\big\},
\qquad
\mathrm{Col}(G_{\mathrm{mACH}})\subseteq\mathcal{V}_{\mathcal{T}}^{0},
\qquad
\mathrm{rank}(G_{\mathrm{mACH}})\le\min\big(C,\ N-N_c\big),
\end{equation}
per iteration. Over training, the visited union of these subspaces remains within the image of the frozen text-embedding manifold, whose effective dimension is measurable offline.
\end{lemma}

\begin{proof}
$S_m^{(i)}$ depends on $x_m$ only through the logits $a_{m,l}^{(i)}$.
The softmax Jacobian
$\partial A_{m,l}^{(i)}/\partial a_{m,j}^{(i)}
=A_{m,l}^{(i)}(\delta_{lj}-A_{m,j}^{(i)})$ gives
\begin{equation}
\nabla_{x_m}S_m^{(i)}
=\frac{1}{\sqrt{C}}\sum_{l,j}c_l^{(i)}A_{m,l}^{(i)}
\big(\delta_{lj}-A_{m,j}^{(i)}\big)\,k_j^{(i)}
=\frac{1}{\sqrt{C}}\sum_{j}\beta_{m,j}^{(i)}k_j^{(i)},
\qquad
\beta_{m,j}^{(i)}
:=A_{m,j}^{(i)}\Big(c_j^{(i)}-\sum_{l}A_{m,l}^{(i)}c_l^{(i)}\Big).
\end{equation}
Since $\sum_{j}A_{m,j}^{(i)}=1$,
\begin{equation}
\sum_{j}\beta_{m,j}^{(i)}
=\sum_{j}A_{m,j}^{(i)}c_j^{(i)}
-\Big(\sum_{j}A_{m,j}^{(i)}\Big)\Big(\sum_{l}A_{m,l}^{(i)}c_l^{(i)}\Big)
=0 .
\end{equation}
In matrix form, with
$\beta_m^{(i)}:=(\beta_{m,1}^{(i)},\dots,\beta_{m,L_i}^{(i)})^\top$ and
$k_j^{(i)}=W_Kt_j^{(i)}$,
\begin{equation}
\nabla_{x_m}S_m^{(i)}
=\frac{1}{\sqrt{C}}\,W_KT^{(i)}\beta_m^{(i)},
\qquad
\mathbf{1}^\top\beta_m^{(i)}=0 .
\end{equation}
The branch gradient $g_m^{(i)}=\delta_m^{(i)}\nabla_{x_m}S_m^{(i)}$
differs only by the scalar BCE residual, so every column of $G^{(i)}$
lies in $\{W_KT^{(i)}\beta:\mathbf{1}^\top\beta=0\}$, a subspace of
dimension at most $L_i-1$. Assembling via \eqref{eq:broadcast_adjoint},
\begin{equation}
\mathrm{Col}(G_{\mathrm{mACH}})
\;\subseteq\;\sum_{i=1}^{N_c}\mathrm{Col}(G^{(i)})
\;\subseteq\;\sum_{i=1}^{N_c}\big\{W_KT^{(i)}\beta:\mathbf{1}^\top\beta=0\big\}
\;=:\;\mathcal{V}_{\mathcal{T}}^{0},
\qquad
\dim\mathcal{V}_{\mathcal{T}}^{0}\le\sum_{i=1}^{N_c}(L_i-1)=N-N_c .
\end{equation}
\end{proof}

\begin{lemma}[Text-free driving signal]
\label{lem:jepa_free}
The JEPA field is driven by text-free targets: the loss gradients $h_o:=\nabla_{\hat{z}_{\mathrm{pred},o}}\mathcal{L}_{\mathrm{J}}$ point toward EMA targets $\mathcal{P}_{\mathrm{EMA}}(x_o)$, functions of $X$ alone \eqref{eq:ema}--\eqref{eq:jepa_loss}. Hence $\mathrm{Col}(G_{\mathrm{J}})$ is not algebraically confined to $\mathcal{V}_{\mathcal{T}}^{0}$, in contrast to Lemmas~\ref{lem:contrast} and \ref{lem:mach}.
\end{lemma}

\begin{proof}
Let $p_o := \hat{z}_{\mathrm{pred},o}$ and $\ell(p,z)$ denote the per-token
loss in \eqref{eq:jepa_loss}, and let
$J_{\bar{z}} := \partial \bar{z}/\partial z = (I-\bar{z}\bar{z}^\top)/\|z\|$
be the sphere-projection Jacobian. The driving signal is
\begin{equation}
h_o \;=\; \nabla_{p_o}\mathcal{L}_{\mathrm{J}}
\;=\; \frac{1}{|\Omega|}\, J_{\bar{p}_o}^\top\,
\Big[ -\bar{z}_{\mathrm{tgt},o} + \beta\, J_{\bar{z}_{\mathrm{tgt},o}}\,
\phi_{\beta}\big(\bar{p}_o-\bar{z}_{\mathrm{tgt},o}\big)\Big] ,
\end{equation}
where $\phi_{\beta}$ is the SmoothL1 derivative applied componentwise.
Since $z_{\mathrm{tgt},o}=\mathcal{P}_{\mathrm{EMA}}(x_o)$, every term in
$h_o$ is a function of $(X,\theta)$ alone. Chain rule through
\eqref{eq:block_attn} then gives the field on unmasked tokens:
\begin{equation}
G_{\mathrm{J}}\big|_{\Omega^c}
\;=\; \underbrace{(W_V^{(v)})^\top\!\! \sum_{o\in\Omega} h_o\,
(\mathcal{A}_{vv,\,o,\cdot})^\top}_{\text{value path}}
\;+\; \underbrace{\sum_{o\in\Omega} \Big(\frac{\partial
\mathcal{A}_{o,\cdot}}{\partial Z}\Big)^{\!\top}
\big(h_o\, (p_o^{(\mathrm{val})})^\top\big)}_{\text{affinity path}} ,
\end{equation}
with $p_o^{(\mathrm{val})}$ the value vectors aggregated in
\eqref{eq:block_attn}. Both terms, and hence
$G_{\mathrm{J}}^{(x)}=J_\theta^\top G_{\mathrm{J}}$, are functions of
$(X,\Omega,\epsilon,\theta,\phi)$. In contrast to
Lemmas~\ref{lem:contrast} and \ref{lem:mach}, no term is a structured
linear combination of text vectors: $\mathrm{Col}(G_{\mathrm{J}})$ is not
algebraically confined to $\mathcal{V}_{\mathcal{T}}^{0}$.
\end{proof}

Write $G_{\mathrm{m}}:=G(\mathcal{L}_{\mathrm{m}})$ and
$G_{\mathrm{J}}:=G(\mathcal{L}_{\mathrm{J}})$, so that the joint field of
$\mathcal{L}_{\mathrm{tot}}$ is
\begin{equation}
G \;=\; G_{\mathrm{m}}+\alpha G_{\mathrm{J}} .
\end{equation}
Centering the JEPA field over the mask noise defines the fluctuation and
its second moment,
\begin{equation}
\label{eq:fluctuation}
\xi_{\mathrm{J}}
\;:=\;
G_{\mathrm{J}}-\mathbb{E}_{\Omega,\epsilon}\big[G_{\mathrm{J}}\,\big|\,X,T\big],
\qquad
\Gamma_{\mathrm{J}}
\;:=\;
\mathbb{E}\big[\xi_{\mathrm{J}}\xi_{\mathrm{J}}^\top\big]\succeq0 .
\end{equation}
The conditional mean is absorbed by the deterministic discriminative
component; $\Gamma_{\mathrm{J}}$ measures the part of the JEPA signal
that varies under mask resampling and therefore cannot cancel against
$G_{\mathrm{m}}$.

\begin{assumption}[Nondegenerate JEPA fluctuation]
\label{ass:nondeg}
The mask-noise fluctuation $\xi_{\mathrm{J}}$ is not almost surely
confined to any hyperplane: for every unit $u$,
$\mathbb{E}[(u^\top\xi_{\mathrm{J}})^2] > 0$.
\end{assumption}

\begin{lemma}[Full-support spectral floor]
\label{lem:jepa_full}
Under Assumption~\ref{ass:nondeg},
$\Gamma_{\mathrm{J}}\succeq c_{\mathrm{J}}I$ with
$c_{\mathrm{J}}:=\lambda_{\min}(\Gamma_{\mathrm{J}})>0$, and the joint
field of $\mathcal{L}_{\mathrm{tot}}$ satisfies
\begin{equation}
k^\top\,\mathbb{E}\big[GG^\top\big]\,k
\;\ge\;\alpha^{2}\,k^\top\Gamma_{\mathrm{J}}k
\;\ge\;\alpha^{2}c_{\mathrm{J}}
\qquad\text{for every unit } k .
\end{equation}
\end{lemma}

\begin{proof}
By definition $\Gamma_{\mathrm{J}}\succeq0$, and
$u^\top\Gamma_{\mathrm{J}}u=\mathbb{E}[(u^\top\xi_{\mathrm{J}})^2]>0$
for every unit $u$ by Assumption~\ref{ass:nondeg}; hence
$\Gamma_{\mathrm{J}}\succ0$ and $c_{\mathrm{J}}>0$.
For the joint field, conditional on $(X,T)$ the vector
$k^\top(G_{\mathrm{m}}+\alpha\,\mathbb{E}_{\Omega,\epsilon}G_{\mathrm{J}})$
is constant and $\mathbb{E}_{\Omega,\epsilon}[\xi_{\mathrm{J}}\,|\,X,T]=0$, so
\begin{equation}
k^\top\mathbb{E}\big[GG^\top\big]k
=\mathbb{E}_{X,T}\Big[
\underbrace{\big\|k^\top\big(G_{\mathrm{m}}+\alpha\,\mathbb{E}_{\Omega,\epsilon}G_{\mathrm{J}}\big)\big\|^2}_{\ge\,0}
+\alpha^{2}\,k^\top\mathrm{Cov}_{\Omega,\epsilon}\big(G_{\mathrm{J}}\,\big|\,X,T\big)\,k\Big]
\ge\alpha^{2}k^\top\Gamma_{\mathrm{J}}k
\ge\alpha^{2}c_{\mathrm{J}} .
\end{equation}
The cross term vanishes since $\mathbb{E}_{\Omega,\epsilon}[\xi_{\mathrm{J}}\,|\,X,T]=0$.
\end{proof}

\subsection{Comparison}
\label{subsec:comparison}

\begin{theorem}[Capacity comparison]
\label{thm:comparison}
At steady state, the reachable gradient subspaces of the three objectives are $\mathrm{span}\{t^{(i)}\}$, $\mathcal{V}_{\mathcal{T}}^{0}$, and $\mathbb{R}^{C}$, with dimension budgets
\begin{equation}
\label{eq:ladder_app}
\dim\le N_c
\quad(\text{contrastive});\qquad
\dim\le N-N_c
\quad(\texttt{mACH});\qquad
\dim=C\ \text{a.s.}
\quad(\text{dual-stream}) .
\end{equation}
Whenever $N_c<N/2$ and $N-N_c<C$ (both hold in our data), the budgets are strictly increasing, forming the ladder $N_c<N-N_c<C$. Consequently, only the dual-stream objective is almost surely free of alignment-blind directions---the only paradigm that preserves alignment capacity in every direction.
\end{theorem}

\begin{proof}
For the two discriminative heads, Lemmas~\ref{lem:contrast} and
\ref{lem:mach} confine the gradient fields to the stated subspaces; every
direction in the complements then satisfies the premise of
Lemma~\ref{lem:decay} and decays to zero capacity.
For the dual-stream objective, write $\Gamma:=\mathbb{E}[GG^\top]\succeq0$
and note that for any unit $k$,
\begin{equation}
k\in\mathrm{null}(\Gamma)
\;\Longleftrightarrow\;
k^\top\Gamma k=\mathbb{E}\big\|k^\top G\big\|^2=0
\;\Longleftrightarrow\;
k^\top G=\mathbf{0}\ \ \text{a.s.},
\end{equation}
so the premise of Lemma~\ref{lem:decay} holds exactly on
$\mathrm{null}(\Gamma)$. Lemma~\ref{lem:jepa_full} gives
\begin{equation}
\Gamma\succeq\alpha^{2}c_{\mathrm{J}}I\succ0
\;\Longleftrightarrow\;
\mathrm{Col}(\Gamma)=\mathbb{R}^{C},
\qquad
\mathrm{null}(\Gamma)=\{\mathbf{0}\}.
\end{equation}
Hence the reachable gradient subspace has dimension $C$, and no
direction satisfies the decay premise at any iteration; under the
sustained-excitation mechanism of Section~\ref{subsec:dynamics}, every
direction retains positive capacity.
\end{proof}

\begin{corollary}[Representation diversity]
\label{cor:diversity}
Since $\Xi_X\succeq0$,
\begin{equation}
\mathrm{cap}(k)=0
\;\Longleftrightarrow\;
k^\top\Xi_Xk=0
\;\Longleftrightarrow\;
k\in\mathrm{null}(\Xi_X),
\qquad
\mathrm{rank}(\Xi_X)=\dim(\text{alignment-active subspace}).
\end{equation}
At steady state, the feature rank therefore obeys the ladder of
Theorem~\ref{thm:comparison}: $\mathrm{rank}(\Xi_X)\le N_c$ under
contrastive training, $\le N-N_c$ under \texttt{mACH}, and $=C$ under
the dual-stream objective (Assumption~\ref{ass:nondeg}).
\end{corollary}

\begin{corollary}[Novel expressions]
\label{cor:ood}
If an out-of-distribution expression produces a key direction $k_{\perp}$
outside the reachable subspace visited during training, then at steady
state
\begin{equation}
\mathrm{cap}_{\mathrm{contrast}}(k_{\perp})=\mathrm{cap}_{\mathrm{mACH}}(k_{\perp})=0,
\qquad
\mathrm{cap}_{\mathrm{tot}}(k_{\perp})>0
\quad(\text{Assumption~\ref{ass:nondeg}}),
\end{equation}
the first two by Lemma~\ref{lem:decay}, the last by
Theorem~\ref{thm:comparison}: such expressions are alignment-blind for
both discriminative-only heads and retain positive capacity under the
dual-stream objective.
\end{corollary}

\section*{Appendix C: Curation Pipeline, Prompt Specifications, and Empirical Properties of O365-Caption}
\label{sec:appendix_o365_details}

In this section, we present the comprehensive engineering specifications, algorithmic boundaries, structural system prompts, and deep empirical properties that characterize the curation of the \texttt{O365-Caption} dataset. To facilitate full reproducibility and provide an exhaustive disclosure of our data-side contributions, our exposition is organized into three complementary subsections: 
(1) \textbf{The Three-Stage Generative Pipeline}, detailing the algorithmic taxonomy mapping and automated geometric guardrails; 
(2) \textbf{Prompt Specifications}, providing the exact, unedited English and Chinese systemic prompt templates used to instruct our Multimodal Large Language Models (MLLMs); and 
(3) \textbf{Linguistic and Statistical Properties}, delivering a global quantitative deconstruction of our dataset's vocabulary density, syntax length, and part-of-speech distributions against standard visual grounding baselines. 

Through this multi-dimensional documentation, we demonstrate how \texttt{O365-Caption} systematically upgrades the rigid, discrete single-word category tags of legacy detection assets into high-density, context-aware, and open-vocabulary referring expressions suitable for industrial-scale generalist pre-training.

\begin{figure*}[!htbp]
    \centering
    \includegraphics[width=0.8\linewidth]{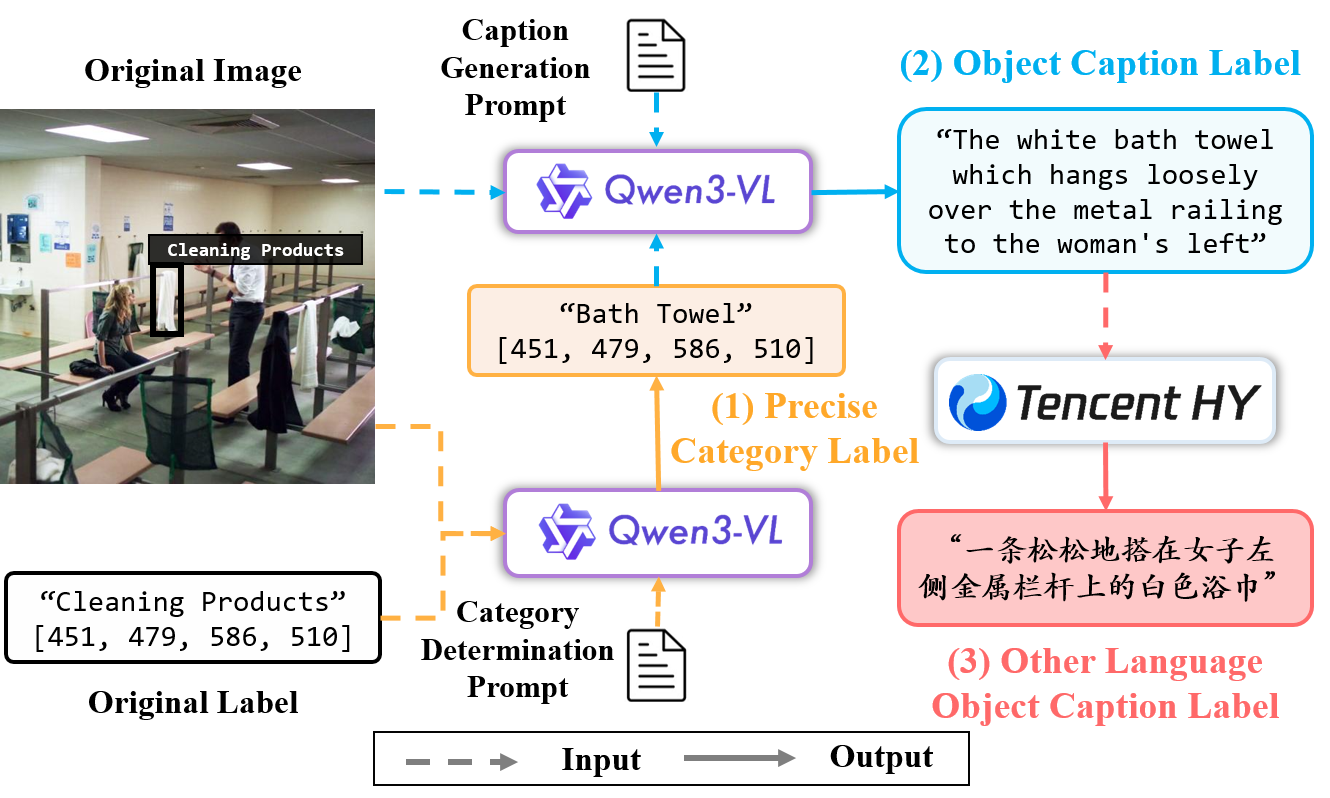}
    \caption{\textbf{The detailed three-stage data generation pipeline of O365-Caption.}(An enlarged detailed view of \emph{Fig.~3(a)} from the main text.) The framework sequentially executes coarse-to-fine semantic disambiguation via Qwen3-VL-2B, dense context-aware object-level captioning via Qwen3-VL-32B, and cross-lingual translation via the Tencent HY-MT-1.5-7B engine.}
    \label{fig:o365pipeline}
\end{figure*}

\subsection{The Pipeline of Constructing Objects365-Caption}

Crucially, a definitive advantage of this three-stage pipeline is its dataset-agnostic versatility; the underlying framework is completely decoupled from Objects365-specific topologies. It can be seamlessly deployed across virtually any legacy object detection asset (such as COCO, OpenImages, or LVIS) to systematically upgrade rigid, closed-set categorical annotations into rich, open-vocabulary grounding tokens without requiring expensive manual re-labeling. The transformation workflow is strictly compartmentalized into three sequential stages, accompanied by automated geometric guardrails.

\textbf{Stage 1: Coarse-to-Fine Semantic Disambiguation} 
The original vocabulary of Objects365 frequently relies on highly coarse or generalized container tags (e.g., ``Cleaning Products'' or ``Other Shoes'') to optimize manual labeling throughput. Directly optimizing open-vocabulary heads on such fuzzy semantic anchors introduces substantial cross-modal gradient contradictions. To resolve this, we leverage a lightweight yet efficient Qwen3-VL-2B model~\cite{bai2025qwen3} to perform visual semantic disambiguation. 

Conditioned on the original image, target bounding box coordinates $[x_{\min}, y_{\min}, x_{\max}, y_{\max}]$, and a specialized \textit{Category Determination Prompt}, the model dynamically inspects the localized region to identify the precise subordinate taxonomy. For instance, as visualized in Fig.~\ref{fig:o365pipeline}, the coarse tag ``Cleaning Products'' is accurately mapping onto the granular classification ``Bath Towel''. 

Critically, Multimodal Large Language Models (MLLMs) inherently suffer from performance degradation and low classification accuracy when processing severely downscaled visual patches. To safeguard our pipeline against small-object hallucinations, we instantiate a rigid spatial filtering threshold based on the relative bounding box coverage:
\begin{equation}
\gamma = \frac{\text{Area}_{\text{bbox}}}{\text{Area}_{\text{image}}} = \frac{(x_{\max} - x_{\min}) \times (y_{\max} - y_{\min})}{H_{\text{image}} \times W_{\text{image}}}
\label{eq:area_ratio}
\end{equation}
For any target instance where $\gamma < 0.05\%$, the semantic refinement stage is bypassed, and the pipeline safely falls back to the original coarse label. This defensive engineering constraint protects data cleanliness but structurally explains the conservative unique caption ratio (UCR) observed in the small-object slices of the final corpus.

\textbf{Stage 2: Context-Aware Caption Generation}
Once the precise taxonomy is anchored, the instance is forwarded to a powerful Qwen3-VL-32B model tasked with dense language synthesis. This stage introduces a dedicated \textit{Caption Generation Prompt} that instructs the model to construct a rich, continuous referring expression by combining three essential dimensions: (1) the purified fine-grained category name from Stage 1, (2) localized visual attributes such as color, material, and state, and (3) relative spatial dynamics with surrounding objects. 

The model outputs highly descriptive, context-aware phrases, successfully upgrading a simple noun phrase into an open-vocabulary grounding target (e.g., transforming ``Bath Towel'' into ``The white bath towel which hangs loosely over the metal railing to the woman's left''). This extensive linguistic compositionality injects the exact high-entropy variations required to prevent the task head's gradient field from collapsing into low-rank sub-manifolds.

\textbf{Stage 3: Cross-Lingual Extension via Translation}
To scale the capacity of our generalist grounding architecture across global linguistic boundaries, the generated English captions undergo an automated multi-lingual expansion. We employ the advanced Tencent HY-MT-1.5-7B translation engine~\cite{zheng2025hy} to convert the synchronized object-level descriptions into target languages, such as Chinese, while preserving exact syntactic mapping and spatial indices. This step transforms the corpus into a parallelized cross-lingual visual grounding benchmark, ensuring the mACH alignment weights remain robust under diverse multi-lingual prompt distributions without requiring test-time structural adaptation.

\textbf{Post-Processing Quality Control}
Following the three-stage generation, a non-parametric quality filter is executed to strip away textual noise. We discard any generated descriptions that contain dangling pronouns or fail to exceed a baseline length of three tokens. Through this data-model co-design pipeline, \texttt{O365-Caption} establishes a robust, highly compositional, and structurally clean data asset comprising 10.0M precise spatial-textual alignment targets.


\subsection{Prompt Specifications}
\begin{tcolorbox}[
    width=\textwidth,
    colback=white,       
    colframe=blue!50!gray,      
    coltitle=white,             
    title=\textbf{Stage 1: Category Determination Prompt},
    fonttitle=\sffamily\small,
    fontupper=\small,
    arc=4mm,                    
    boxrule=0.5mm,              
    left=6mm, right=6mm, top=4mm, bottom=4mm 
]
\textbf{\large Role}

\vspace{1mm}
You are a fine-grained visual classifier and semantic disambiguation assistant, specialized in identifying the exact sub-category of an object within a designated bounding box.

\vspace{3mm}
\textbf{\large Inputs}
\begin{itemize}
    \item \textbf{Image:} The full image.
    \item \textbf{Bounding Box:} [x1, y1, x2, y2]
    \item \textbf{Coarse Category Label:} The initial broad category name from Objects365 (e.g., "Cleaning Products", "Other Shoes", "Vehicle").
\end{itemize}

\vspace{3mm}
\textbf{\large Core Task}

\vspace{1mm}
Examine the region defined by the bounding box. Perform semantic disambiguation and fine-grained refinement on the provided Coarse Category Label based on the visual context. Determine the precise, specific name of the target object (e.g., refine "Cleaning Products" to "Bath Towel", "Other Shoes" to "Boots", or "Vehicle" to "SUV").

\vspace{3mm}
\textbf{\large Output Rules (Strictly Enforced)}
\begin{enumerate}
    \item Return ONLY the refined, fine-grained category name (a single noun or short noun phrase).
    \item Do NOT include any quantifiers, articles ("a", "an", "the"), descriptive adjectives, punctuation, or conversational explanations.
    \item If the object is too small, blurry, or heavily obscured to be further specified, return the original [Coarse Category Label] exactly.
\end{enumerate}

\vspace{3mm}
\textbf{\large Examples}
\begin{itemize}
    \item \textbf{Input:} Box=[451, 479, 586, 510], Coarse Category=Cleaning Products\\
          \textbf{Output:} Bath Towel
    \item \textbf{Input:} Box=[100, 200, 150, 250], Coarse Category=Other Shoes\\
          \textbf{Output:} Boots
    \item \textbf{Input:} Box=[50, 50, 70, 70], Coarse Category=Bird (Too small to specify the species)\\
          \textbf{Output:} Bird
\end{itemize}
\end{tcolorbox}

\vspace{4mm} 

\begin{tcolorbox}[
    width=\textwidth,
    colback=white,       
    colframe=blue!50!gray,      
    coltitle=white,             
    title=\textbf{Stage 2: Caption Generation Prompt},
    fonttitle=\sffamily\small,
    fontupper=\small,
    arc=4mm,                    
    boxrule=0.5mm,              
    left=6mm, right=6mm, top=4mm, bottom=4mm 
]
\textbf{\large Role}

\vspace{1mm}
You are a precise visual grounding description assistant, specialized in generating standardized, open-vocabulary phrases for specific target objects within a designated bounding box.

\vspace{3mm}
\textbf{\large Inputs}
\begin{itemize}
    \item \textbf{Image:} The full image.
    \item \textbf{Bounding Box:} [x1, y1, x2, y2]
    \item \textbf{Refined Target Category:} The precise category name purified from Stage 1 (e.g., "SUV", "Bath Towel", "Tree").
\end{itemize}

\vspace{3mm}
\textbf{\large Core Task}
\begin{enumerate}
    \item \textbf{Attribute and Context Extraction:} Observe the target within the bounding box, capturing its salient visual attributes (color, material, state) and its spatial relationships relative to the surrounding environment.
    \item \textbf{Instance Counting:} Examine the bounding box region carefully to determine the density or distribution of the specified target category.
\end{enumerate}

\vspace{3mm}
\textbf{\large Output Rules (Strictly Enforced)}
\begin{enumerate}
    \item \textbf{Format Specification:}
    \begin{itemize}
        \item Output ONLY a single, continuous descriptive noun phrase in the format: "[Quantifier] + [Descriptive Attributes/Context] + [Refined Target Category]".
        \item Do NOT include any punctuation (periods, quotation marks), line breaks, or conversational explanations.
        \item \textbf{Correct:} a red SUV, several roadside green trees
        \item \textbf{Incorrect:} A red umbrella. (Contains a period), There are three umbrellas in this box (Contains explanation)
    \end{itemize}
    
    \item \textbf{Quantifier Constraints:}
    \begin{itemize}
        \item If Count = 1: Must use a standard singular quantifier (e.g., "a", "an", "the", or "one").
        \item If Count > 1: \textbf{Never use exact numbers}. You must use indefinite or vague plural quantifiers (e.g., "multiple", "several", "a row of", "a group of").
    \end{itemize}
    
    \item \textbf{Spatial Context and Partial Cropping:}
    \begin{itemize}
        \item If the box contains only a \textit{subset} of a group of identical items (e.g., framing only the 2 leftmost cars in a long row): The description must incorporate relative positioning, e.g., "the leftmost white sedans".
        \item If the box encompasses the \textit{entirety} of the visible group: Describe only the attributes and vague quantity, e.g., "several white sedans".
    \end{itemize}
    
    \item \textbf{Uncertainty Fallback:}
    \begin{itemize}
        \item If details or counts are severely obscured, fallback to the vague plural quantifier "multiple" and retain the base [Refined Target Category] name exactly.
    \end{itemize}
\end{enumerate}

\vspace{3mm}
\textbf{\large Examples}
\begin{itemize}
    \item \textbf{Input:} Box=[100,100,300,300], Refined Target=SUV, 1 red SUV inside\\
          \textbf{Output:} a red SUV
    \item \textbf{Input:} Box=[200,200,800,400], Refined Target=Tree, a row of 5 trees inside\\
          \textbf{Output:} several roadside green trees
    \item \textbf{Input:} Box=[0,0,200,200], Refined Target=Apple, a countless pile of red apples inside\\
          \textbf{Output:} multiple red apples
\end{itemize}
\end{tcolorbox}

\subsection{Detailed Linguistic and Statistical Properties of O365-Caption}
\label{subsec:linguistic_properties}

To provide a transparent, macro-level evaluation of the proposed \texttt{O365-Caption} dataset, we conduct a comprehensive statistical and linguistic analysis across four core dimensions: vocabulary scaling behaviors, sentence length distributions, lexical growth patterns, and part-of-speech (POS) compositions. We benchmark our corpus against three prominent grounding and phrase alignment datasets: \texttt{MixGrounding}, \texttt{flickr30k}, and \texttt{RefCOCOg}. The aggregated empirical results are visualized in Fig.~\ref{fig:linguistic_landscape}.

\begin{figure}[t]
    \centering
    \includegraphics[width=0.7\textwidth]{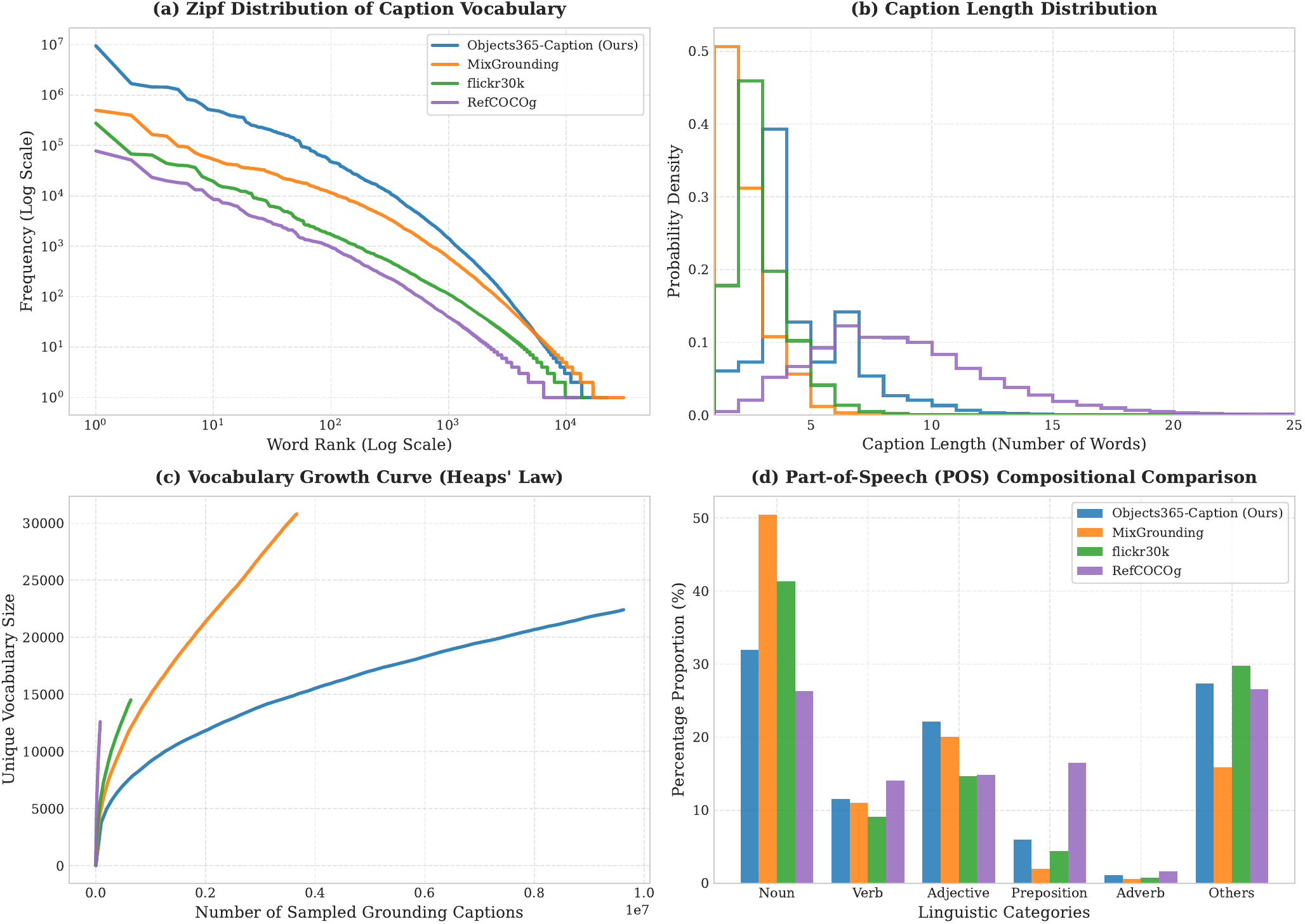}
    \caption{\textbf{Comprehensive linguistic and statistical landscape of O365-Caption compared with prominent baselines.}(An enlarged detailed view of \emph{Fig.~3(b)} from the main text.) (a) Zipf's law distribution tracking absolute word frequencies against vocabulary ranks. (b) Probability density estimation of caption length (word count per expression). (c) Unique vocabulary growth curves as a function of sampled grounding captions (Heaps' Law). (d) Relative percentage composition of granular part-of-speech (POS) tags calculated via the NLTK parser\cite{loper2002nltk}. Zoom in for best view.}
    \label{fig:linguistic_landscape}
\end{figure}


By analyzing the four-quadrant distribution in Fig.~\ref{fig:linguistic_landscape}, we derive the following critical observations regarding the compositionality and scaling capacity of \texttt{O365-Caption}:

\begin{itemize}
    \item \textbf{Zipfian Vocabulary Distribution (Fig.~\ref{fig:linguistic_landscape}a):} In the log-log frequency-rank plane, all examined corpora conform to standard power-law scaling\cite{zipf2016human}. Crucially, our \texttt{O365-Caption} (blue curve) dominates the vertical frequency spectrum, exceeding comparative baselines by $1$ to $2$ orders of magnitude with an absolute peak frequency approaching $10^7$. Moreover, its tail extends continuously toward the high-rank open-vocabulary domain without sharp truncation, proving that our dataset effectively injects massive scaling volume while maintaining a structurally healthy long-tail vocabulary distribution.
    
    \item \textbf{Density Alignment of Caption Lengths (Fig.~\ref{fig:linguistic_landscape}b):} The probability density estimation unveils sharp topological discrepancies in expression length. While \texttt{MixGrounding} collapses into a sparse noun-phrase peak ($1\sim2$ words) and \texttt{RefCOCOg} exhibits an over-dispersed, redundant long-tail stretching beyond $15$ words, \texttt{O365-Caption} exhibits an idealized, tightly bound bell-shaped distribution peaking precisely at $4$ words. This structure ensures high linguistic density and syntactic conciseness, filtering out irrelevant narrative noise while supplying optimal token-level cross-modal information for standard grounding detection heads.
    
    \item \textbf{Sub-linear Vocabulary Growth (Fig.~\ref{fig:linguistic_landscape}c):} Evaluating through the lens of Heaps' Law ($V = K \cdot N^\beta$)\cite{heaps1978information}, our dataset provides an unprecedented horizontal data span, scale-stretching to over $1.0 \times 10^7$ training descriptions. Although the unique vocabulary size displays a natural sub-linear asymptotic behavior as it approaches a saturation bound of over $22,500$ unique tokens, the growth trajectory retains a steady, upward-sloping derivative. This persistent elasticity serves as a critical data-side safeguard, guaranteeing that the model continually encounters novel syntactic patterns during large-scale pre-training, thereby fundamentally mitigating representation collapse.
    
    \item \textbf{Deconstruction of Noun Monopoly via POS Matrix (Fig.~\ref{fig:linguistic_landscape}d):} The part-of-speech composition reveals the deeper grammatical diversity of our corpus. Traditional automated visual-grounding pairs suffer heavily from a \textit{noun monopoly}, where rigid object categories dominate the text distribution (e.g., \texttt{MixGrounding} exhibits over $50\%$ noun density). In contrast, \texttt{O365-Caption} successfully dilutes the noun ratio to a balanced $32\%$, while significantly lifting the proportion of \textbf{Adjectives ($22\%$)} and maintaining structural distributions of \textbf{Verbs} and \textbf{Prepositions}. This high density of modifiers and spatial markers empirically verifies that our dataset is richly populated with fine-grained visual attributes and target-context relational descriptors, unlocking robust open-vocabulary grounding capability.
\end{itemize}

\subsection{Qualitative Examples of O365-Caption}
\label{app:caption_examples}

Figure~\ref{fig:o365_caption_examples} presents representative annotations from O365-Caption. Unlike conventional grounding datasets that typically associate each object with a short noun phrase, \texttt{O365-Caption} substantially enriches the linguistic diversity while preserving the original Objects365 bounding-box annotations.

\begin{figure}[!t]
    \centering
    \begin{subfigure}[t]{1\textwidth}
        \centering
        \includegraphics[width=\linewidth]{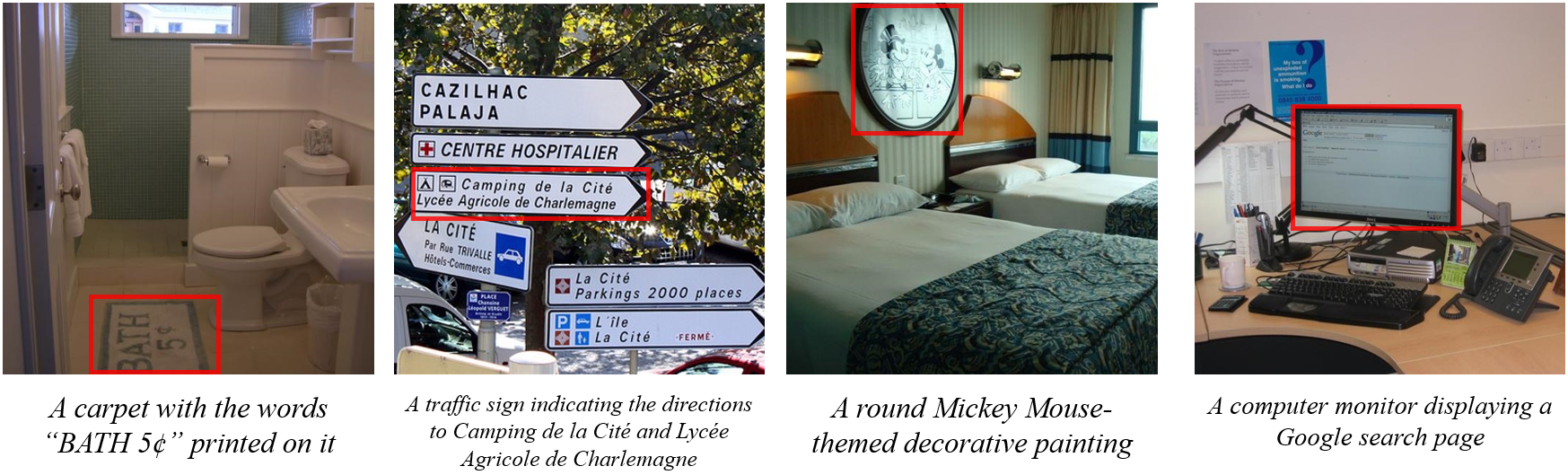}
        \caption{Examples containing scene text, multilingual expressions, and OCR-related grounding targets.}
        \label{fig:visuala}
    \end{subfigure}
    
    \begin{subfigure}[t]{1\textwidth}
        \centering
        \includegraphics[width=\linewidth]{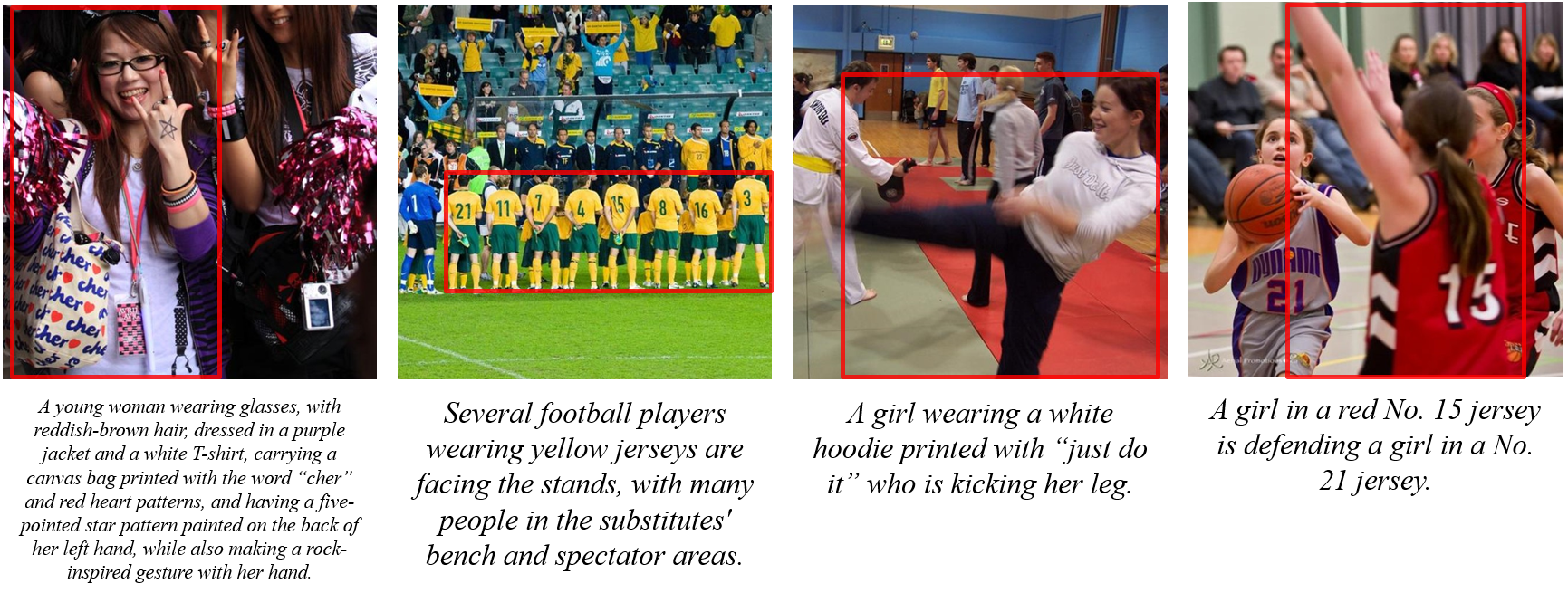}
        \caption{Examples illustrating long-form descriptions, attribute composition, and multi-object reasoning.}
        \label{fig:visualb}
    \end{subfigure}
    \caption{Representative annotation examples from O365-Caption.
    Compared with conventional grounding datasets, our annotations contain richer compositional descriptions, long-form referring expressions, multi-object reasoning, and naturally preserved scene text (including multilingual and non-ASCII content), while retaining the original Objects365 bounding-box annotations.}
    \label{fig:o365_caption_examples}
\end{figure}

Specifically, the generated expressions exhibit several desirable properties:

\begin{itemize}
    \item \textbf{Rich semantic descriptions.}
    Objects are described using complete natural language expressions instead of isolated category names, providing substantially stronger compositional supervision.

    \item \textbf{Long-form referring expressions.}
    Many annotations include appearance, attributes, actions, and surrounding context, producing significantly longer descriptions than existing REC datasets.

    \item \textbf{Multi-object reasoning.}
    Expressions frequently describe groups of objects (e.g., ``Several football players wearing yellow jerseys...'') rather than a single isolated instance, encouraging relational understanding.

    \item \textbf{Scene text and multilingual content.}
    \texttt{O365-Caption} naturally preserves textual content appearing in images, including signage, trademarks, and non-English words (e.g., ``Camping de la Cité''), which introduces numerous non-ASCII tokens (e.g. ``5\textcent'') rarely covered by previous grounding datasets.
\end{itemize}

These examples illustrate the substantially richer linguistic supervision provided by \texttt{O365-Caption}, complementing the quantitative vocabulary statistics presented in Section Methodology.

\section*{Appendix D: Annotation Refinement for Existing Grounding Datasets}
\label{app:annotation_refinement}

Although the primary objective of our annotation pipeline is to construct the proposed \texttt{O365-Caption} dataset, the pipeline itself is general and can also be applied to improve existing automatically constructed grounding corpora. To investigate this capability, we analyze the widely used MixedGrounding (GoldG) dataset, which serves as the pre-training corpus for several open-vocabulary grounding models.

Unlike human-written referring expressions, MixedGrounding is assembled by automatically merging heterogeneous vision-language datasets. Consequently, we observe that a considerable portion of its annotations are semantically incomplete or linguistically ambiguous. Typical examples include isolated pronouns (e.g., ``he'', ``she'', ``his'', ``they''), standalone attributes (e.g., ``blue'', ``white'', ``red''), truncated noun phrases (e.g., ``back'', ``side'', ``section''), generic object references (e.g., ``object'', ``item'', ``someone''), and malformed text fragments (e.g., ``wh'', ``aprt'', ``are t''). Although these expressions may correspond to valid image regions, they generally fail to uniquely identify an object and therefore provide weak supervision for referring expression grounding.

To quantify this phenomenon, Table~\ref{tab:noisy_caption_statistics} summarizes representative noisy grounding expressions that appear more than 100 times in MixedGrounding. The high frequency of these low-information annotations indicates that the problem is systematic rather than occasional.

Our annotation pipeline automatically detects such low-quality expressions and reconstructs them into complete, context-aware referring descriptions while preserving the original bounding-box annotations. Specifically, the repaired captions enrich the original annotations with explicit object semantics, visual attributes, and contextual cues, substantially increasing their linguistic informativeness without modifying the associated localization labels. 

These observations suggest that our annotation pipeline is not limited to constructing \texttt{O365-Caption}. More generally, it serves as a scalable annotation refinement framework that can systematically improve the linguistic quality of existing grounding datasets while maintaining their original spatial supervision.

\begin{table}[t]
\centering
\caption{Representative low-information grounding expressions occurring more than 100 times in MixedGrounding (GoldG).}
\label{tab:noisy_caption_statistics}
\begin{tabular}{ll}
\toprule
Category & Examples \\
\midrule
Pronouns &
he, she, his, him, they \\
Attributes &
blue, white, red, gray, silver \\
Incomplete phrases &
back, side, section, end, view \\
Generic references &
object, item, someone, lot \\
Malformed fragments &
wh, aprt, are t \\
\bottomrule
\end{tabular}
\end{table}

\section*{Appendix E: Implementation and Hyperparameter Details}
\label{sec:appendix_hyperparameters}

To ensure complete empirical transparency and guarantee absolute reproducibility, we provide the exhaustive configuration profiling, optimization strategies, and hardware deployment metrics utilized across our entire pipeline. The training execution is managed via a centralized hyperparameter protocol and parallelized across a high-capacity distributed clusters. The precise operational layout across both CNN-based and DETR-based variants is structured in Tab.~\ref{tab:hyperparameter_specifications}.

\begin{table}[htbp]
\centering
\caption{\textbf{Comprehensive hyperparameter and training configuration layout.} This unified profile governs both the default generalist training protocol and benchmark-specific downstream adaptation variants across different architectural backbones.}
\label{tab:hyperparameter_specifications}
\resizebox{\linewidth}{!}{%
\begin{tabular}{ll|ll}
\toprule
\textbf{Configuration Layer} & \textbf{Hyperparameter Item} & \textbf{Default Operational Value} & \textbf{Hardware / Library Anchor} \\
\midrule
\textit{(a) Infrastructure \& Base} 
& Model Identifier & \texttt{CNN} / \texttt{DETR} & PyTorch Pipeline Framework \\
& Visual Input Resolution & $640 \times 640$ pixels & Downstream Edge-friendly Deployment \\
& Total Pre-training Volume & 30 epochs & Training Iteration Boundary \\
& Base Learning Rate ($\eta$) & $2.0 \times 10^{-3}$ (CNN) / $1.0 \times 10^{-4}$ (DETR) & Linear Scaling Topology Rule \\
& Weight Decay Metric & 0.025 (CNN) / $1.0 \times 10^{-4}$ (DETR) & $L_2$ Regularization Boundary \\
& Learning Rate Schedule & Cosine Annealing (CNN) / Step Decay (DETR) & Optimization Convergence Path \\
& Warmup Optimization Duration & 3 epochs / 1 epoch & Linear Learning Rate Step Ascent \\
& Mixed-Precision Setting & Floating-Point 16 (AMP) & \texttt{torch.cuda.amp} Acceleration \\
\midrule
\textit{(b) Hardware, System \& Cache}
& Pre-training Hardware Profile & $8 \times$ NVIDIA RTX PRO 6000 (96GB PCIe) GPUs & Distributed Cluster Topology \\
& Data Cache Architecture & Serialization Pickle & Persistent Meta-storage Disk Cache \\
& Per-GPU Batch Capacity & 16 samples & Micro-batch Dimension Boundary \\
& Dataloader Throughput & 4 workers per GPU & Multi-threaded Asynchronous Prefetching \\
\midrule
\textit{(c) Model \& Representation}
& Visual Backbone Scale & [Tiny-Small-Base] & DINOv3 Self-Supervised Initialization \\
& Primary Language Encoder & Frozen \texttt{Qwen3-VL-Embedding-2B} & Pristine Multilingual Latent Space Anchor \\
& Hidden Latent Task Dimension & [768-768-1024] channels & Multi-modal Cross-Attention Gating Map \\
& InfoNCE Categorical Batch Constraints & $N_{\mathrm{class}} = 20$ blocks & Contrastive Alignment Gating Vector \\
\midrule
\textit{(d) Deployment \& Inference Efficiency}
& Evaluation Inference Hardware & Single NVIDIA V100 (32GB) GPU & Standardized Evaluation Node \\
& Latency (Fixed Text / Cached Embeddings) & \textbf{33.2 ms} / image & Multi-query Category Decoupling Pass \\
& Latency (Dynamic Text / On-the-fly) & \textbf{140.5 ms} / image & Real-time End-to-end REC Tracking \\
& Peak Training VRAM Footprint & $\sim$90.3 GB per GPU & Memory-capped Scalable Clustering \\
& Peak Inference VRAM Footprint & $\sim$1.4 GB & Low-resource Edge Node Constraint \\
\bottomrule
\end{tabular}
}
\end{table}

\paragraph{Optimization and Convergence Controls.} 
All hyperparameters are frozen deterministically across the cross-dataset evaluation matrix to validate generalist transfer stability. The optimization trajectory utilizes the AdamW optimizer paired with variant-specific decay behaviors, under standard numerical coefficients ($\beta_1=0.9, \beta_2=0.95, \epsilon=10^{-8}$). Specifically, the CNN-based variant is optimized with an initial learning rate of $2.0 \times 10^{-3}$ and weight decay of 0.025 regulated by a cosine annealing scheduler with a 3-epoch linear warmup. Concurrently, the DETR-based variant adopts an initial learning rate of $1.0 \times 10^{-4}$ and weight decay of $1.0 \times 10^{-4}$, governed by a step decay schedule activated at 80\% and 90\% of the total milestone, utilizing a 1-epoch warmup duration to handle Transformer-based structural updates. To support the multi-modal text-conditioned dense representations during pre-training, the model maximizes the memory capacity of an 8-GPU NVIDIA RTX PRO 6000 cluster, maintaining a stable $\sim$90.3 GB VRAM allocation per node.

\paragraph{Decoupled Latency Analysis.}
The visual encoder is initialized using the self-supervised \texttt{DINOv3-ConvNeXt-Tiny} weights pre-trained on the LVD-1689M corpus, while the language encoder routes representations from the pristine multilingual \texttt{Qwen3-VL-Embedding-2B} repository. Crucially, as profiled in Tab.~\ref{tab:hyperparameter_specifications}(d), our framework exhibits distinctive operational advantages under varying deployment conditions when benchmarked on a standard NVIDIA V100 (32GB) evaluation GPU. 

When evaluated under the \textit{Dynamic Text} setting (standard in sequential online REC tasks), the language encoder must be invoked on-the-fly for each text query, yielding an end-to-end processing latency of \textbf{140.5 ms} per image. This empirical baseline highlights the massive computational bottleneck imposed by running large-scale linguistic models forward passes in real-time tracking loops. In sharp contrast, under the \textit{Fixed Text} configuration (standard in vocabulary-fixed open-vocabulary detection settings), the high-capacity linguistic representations of the frozen text encoder are pre-cached. This entirely bypasses the LLM extraction overhead during runtime, allowing the decoupled visual stream and the broadcast-based \texttt{mACH} layer to process the image frame in only \textbf{33.2 ms} (saving $\sim$76.3\% computation time). This dramatic latency reduction combined with a lightweight inference VRAM footprint of only $\sim$3.4 GB empirically validates the framework's architecture-agnostic suitability for high-throughput, edge-deployable real-time visual grounding applications.

\section*{Appendix F: Limitations}
\label{app:limitations}

Our theoretical analysis characterizes the representational properties of the learned visual features rather than end-task grounding accuracy. Specifically, the proposed notion of directional alignment capacity explains which semantic directions remain available for vision-language alignment, but does not account for optimization dynamics, supervision quality, or score calibration, all of which also affect empirical performance.

Furthermore, while the dual-stream objective eliminates alignment-blind directions and preserves representation diversity, it does not explicitly address confidence calibration under distribution shift. Consequently, out-of-distribution objects may still receive conservative (low-confidence) predictions despite retaining sufficient representational capacity for alignment. Addressing calibration and open-set confidence estimation remains an important direction for future work.

\end{document}